\documentclass[11pt,a4paper,final]{article}
\usepackage{microtype}
\usepackage{graphicx}
\usepackage{subfigure}
\usepackage{enumerate}
\usepackage{amsthm,amsmath,amsfonts}
\usepackage{booktabs} % for professional tables
\usepackage{bm}
\usepackage{diagbox}
\usepackage{thm-restate}
\usepackage[algo2e,ruled,vlined]{algorithm2e}
\usepackage{geometry}
\usepackage{amssymb}
\usepackage{tikz}
\usepackage{url}
\usepackage{setspace}
\usepackage{framed}
\usepackage{xcolor}
\usepackage{soul}
\usepackage{longtable}

\usepackage{times}
\usepackage{enumitem}
\usepackage{varwidth}
\usepackage{wrapfig}
\usepackage{caption}

\usepackage[utf8]{inputenc} % allow utf-8 input
\usepackage[T1]{fontenc}    % use 8-bit T1 fonts
\usepackage{hyperref}       % hyperlinks
\usepackage{url}            % simple URL typesetting
\usepackage{booktabs}       % professional-quality tables
\usepackage{amsfonts}       % blackboard math symbols
\usepackage{nicefrac}       % compact symbols for 1/2, etc.

\usepackage{breakcites}
\usepackage{appendix}
\usepackage{mdwlist}
\usepackage{xspace}
\usepackage{color}
\usepackage{mathrsfs}

\usepackage{comment}
\usepackage{multirow}

\usepackage{cleveref}
\usepackage{authblk}
\usepackage{natbib}

\newcommand{\ind}{\mathbb{I}}

\DeclareFontShape{T1}{ptm}{m}{scit}{<-> ssub * ptm/m/sc}{}

\newcommand{\Appendix}[1]{the full version for}

\newtheorem{theorem}{Theorem}[section]
\newtheorem{lemma}[theorem]{Lemma}
\newtheorem{proposition}[theorem]{Proposition}

\newtheorem{remark}{Remark}

\newtheorem{definition}{Definition}

\newcommand{\R}{\mathbb{R}}

\newcommand{\1}{\mathbf{1}}

\newcommand{\cF}{\mathcal{F}}

\newcommand{\eps}{\epsilon}

\DeclareMathOperator*{\argmax}{argmax}

\title{Tuning Algorithmic and Architectural Hyperparameters in Graph-Based Semi-Supervised Learning with Provable Guarantees}

\author[1]{Ally Yalei Du}
\author[1]{Eric Huang}
\author[2]{Dravyansh Sharma\footnote{Most of the work was done while DS was at CMU.}}

\affil[1]{Carnegie Mellon University
}
\affil[2]{Toyota Technological Institute at Chicago}

\date{}

\begin{document}
\maketitle

\begin{abstract}
Graph-based semi-supervised learning is a powerful paradigm in machine learning for modeling and exploiting the underlying graph structure that captures the relationship between labeled and unlabeled data. A large number of classical as well as modern deep learning based algorithms have been proposed for this problem, often having tunable hyperparameters. 
We initiate a formal study of tuning algorithm hyperparameters from parameterized algorithm families for this problem.
We obtain novel $O(\log n)$ pseudo-dimension upper bounds for hyperparameter selection in three classical label propagation-based algorithm families, where $n$ is the number of nodes, implying bounds on the amount of data needed for learning provably good parameters.
We further provide matching $\Omega(\log n)$ pseudo-dimension lower bounds, thus asymptotically characterizing the learning-theoretic complexity of the parameter tuning problem.
We extend our study to selecting architectural hyperparameters in modern graph neural networks. We bound the Rademacher complexity for tuning the self-loop weighting in recently proposed Simplified Graph Convolution (SGC) networks.
We further propose a tunable architecture that interpolates graph convolutional neural networks (GCN) and graph attention networks (GAT) in every layer, and provide Rademacher complexity bounds for tuning the interpolation coefficient. 
\end{abstract}

\section{Introduction}
Semi-supervised learning is a powerful paradigm in machine learning which reduces the dependence on expensive and hard-to-obtain labeled data, by using a combination of labeled and unlabeled data. This has become increasingly relevant in the era of large language models, where an extremely large amount of labeled training data is needed. A large number of techniques have been proposed in the literature to exploit the structure of unlabeled data, including popularly used graph-based semi-supervised learning algorithms \citep{Blum2001LearningFL,zhu2003semi,zhou2003learning,delalleau2005efficient,chapelle2009semi}. More recently, there has been an increasing interest in developing effective neural network architectures for graph-based learning \citep{kipf2016semi,velivckovic2017graph,iscen2019label}. However, different algorithms, architectures, and values of hyperparameters perform well on different datasets \citep{dwivedi2023benchmarking}, and there is no principled way of selecting the best approach for the data at hand. In this work, we initiate the study of theoretically principled techniques for learning hyperparameters from infinitely large semi-supervised learning algorithm families.

In graph-based semi-supervised learning, the graph nodes consist of labeled and unlabeled data points, and the graph edges denote feature similarity between the nodes. There are several classical ways of defining a graph-based regularization objective that depend on the available and predicted labels as well as the graph structure. Optimizing this objective yields the predicted labels and the accuracy of the predictions depends on the chosen objective. The performance of the same objective may vary across datasets. By studying parameterized families of objectives, we can learn to design the objective that works best on a given domain-specific data. Similarly, modern deep learning based techniques often have several candidate architectures and choices for hyperparameters, often manually optimized for each application domain.
Recent work has considered the problem of learning the graph hyperparameter used in semi-supervised learning \citep{balcan2021data,fatemi2021slaps} but leaves the problem of selecting the hyperparameter wide open.
In this paper, we take important initial steps to build the theoretical foundations of algorithm hyperparameter selection in graph-based semi-supervised learning. 

Note that we focus specifically on algorithm hyperparameters, such as self-loop weights, leaving optimization hyperparameters like learning rates outside the scope of this study.\looseness-1

\subsection{Contributions}

\begin{itemize}[leftmargin=*]
    \item We study hyperparameter tuning in three canonical label propagation-based semi-supervised learning algorithms: the local and global consistency~\citep{zhou2003learning}, the smoothing-based~\citep{delalleau2005efficient}, and a novel normalized adjacency matrix-based algorithm. We prove new $O\left(\log n\right)$ pseudo-dimension upper bounds for all three families, where $n$ is the number of graph nodes. Our proofs rely on a unified template based on determinant evaluation and root-counting, which may be of independent interest.
    \item We provide matching $\Omega\left(\log n\right)$ pseudo-dimension lower bounds for all three aforementioned families. Our proof involves novel constructions of a class of partially labeled graphs that exhibit fundamental limitations in tuning the label propagation algorithms. We note that our lower bound proofs are particularly subtle and technically challenging, and involve the design of a carefully constructed set of problem instances and hyperparameter thresholds that shatter these instances.
    \item Next, we consider the modern graph neural networks (GNNs). We prove a new Rademacher complexity bound for tuning the weight of self-loops for a popular architecture proposed in \cite{wu19simplifying}, the Simplified Graph Networks (SGC). 
    \item We propose an architecture (GCAN) where a hyperparameter $\eta$ is introduced to interpolate two canonical GNN architectures: graph convolutional neural networks (GCNs) and graph attention neural networks (GATs). We bound the Rademacher complexity of tuning $\eta$. Because the parameter dimension is different, the Rademacher complexity of SGC and GCAN has different dependencies on the feature dimension $d$: $\sqrt{d}$ for SGC while $d$ for GCAN. 
    \item We conducted experiments to demonstrate the effectiveness of our hyperparameter selection framework.
\end{itemize}

\subsection{Related Work}
\paragraph{Graph Based Semi-supervised Learning} Semi-supervised learning is a popular machine learning paradigm with significant theoretical interest~\citep{zhou2003learning,delalleau2005efficient,Balcan2010ADM,garg2020generaliz}. Classical algorithms focus on label-propagation based techniques, such as \citet{zhou2003learning}, \citet{zhu2003semi}, and many more. 
In recent years, graph neural networks (GNNs) have become increasingly popular in a wide range of application domains
~\citep{kipf2016semi,velivckovic2017graph,iscen2019label}.
A large number of different architectures have been proposed, including graph convolution networks, graph attention networks, message passing, and so on \citep{dwivedi2023benchmarking}. 
Both label propagation-based algorithms and neural network-based algorithms are practically useful~\citep{Balcan2005PersonII,kipf2016semi}. For example, although GNN-based algorithms are more predominant in applications, \cite{huang2020combininglabelpropagationsimple} show that modifications to label propagation-based algorithms can outperform GNN. For node classification in GNN, many work study generalization guarantees for tuning network weights in GNNs ~\citep{oono2021optimizationgeneralizationanalysistransduction, esser2021learningtheorysometimesexplain, tang2023understandinggeneralizationgraphneural}. In contrast, we study the tuning of the \textit{hyperparameters} related to the GNN architecture. 

\paragraph{Hyperparameter Selection} Hyper-parameters, such as the weight for self-loop, play important roles in the performance of both classical methods and GNNs. 
In general, hyperparameter tuning is performed on a validation dataset, and follows the same procedure: determine which hyperparameters to tune and then search within their domain for the combination of parameter values with best performance \citep{yu2020hyperparameteroptimizationreviewalgorithms}. Many methods are proposed to efficiently search within the parameter space, such as grid search, random search \citep{JMLR:v13:bergstra12a}, and Bayesian optimization (\cite{Mockus1974bayesian}; \cite{Mockus1978application}; \cite{jones1998efficient}). A few existing works investigate the theoretical aspects of these methods, such as through generalization guarantees and complexities of the algorithms. 

A recently introduced paradigm called data-driven algorithm design is useful for obtaining formal guarantees for hyperparameter tuning~\citep{balcan2020data,sharma2024data}, and has found several applications (for instance, \citealt{blum2021learning,bartlett2022generalization,balcan2023analysis,balcan2024subsidy}). In particular, \citet{balcan2024provablytuningelasticnetinstances,Balcan2023NewBF,balcan2025distribution} study the regularization hyperparameter in the ElasticNet  and \cite{balcan2024learning} study learning decision tree algorithms. For unsupervised learning, \citet{balcan2019datadrivenclusteringparameterizedlloyds,balcan2024algorithm} study a parameterized  family of clustering algorithms and study the sample and computational complexity of learning the parameters. For semi-supervised learning, a recent line of work (\cite{balcan2021data}; \cite{sharma2023efficiently}) considers the problem of learning the best graph hyperparameter from a set of problem instances drawn from a data distribution. Another recent work \citep{balcan2025samplecomplexitydatadriventuning} investigates the kernel hyperparameters in GNN architecture, and derives the generalization guarantees through pseudo-dimension. However, no existing work theoretically studies the tuning of the labeling \textit{algorithm hyperparameter} in semi-supervised learning, or investigates data-dependent bounds on hyperparameter selection in deep semi-supervised learning algorithms through Rademacher Complexity. We note that in this work we focus
on the statistical learning setting (i.e. the problem instances are drawn from a fixed, unknown distribution),
but it would be an interesting direction to study online tuning of the hyperparameters using tools from prior
work~\citep{Sharma2019LearningPL,Sharma2024NoIR,Sharma2025OfflinetoonlineHT}.
\section{Preliminaries}

\paragraph{Notations.}
Throughout this paper, $f(n) = O(g(n))$ denotes that there exists a constant $c > 0$ such that $|f(n)| \leq c|g(n)|$. $f(n) = \Omega(g(n))$ denotes that there exists a constant $c > 0$ such that $|f(n)| \geq c|g(n)|$. The indicator function is indicated by $\mathbb{I}$, taking values in $\{0, 1\}$.
In addition, we define the shorthand $[c] = \{1, 2, \ldots, c\}$. For a matrix $W$, we denote its Frobenius norm by $\|W\|_F$ and spectral norm by $\|W\|$. We also denote the Euclidean norm of a vector $v$ by $\|v\|$.\looseness-1 

\paragraph{Graph-based Semi-supervised Learning.}
We are given $n$ data points, where some are labeled, denoted by $L \subseteq [n]$, and the rest are unlabeled. We may also have features associated with each data point, denoted by $z_i \in \mathbb{R}^d$ for $i \in [n]$. We can construct a graph $G$ by placing (possibly weighed) edges $w(u,v)$ between pairs of data points $u,v$. The created graph $G$ is denoted by $G = (V, E)$, where $V$ represents the vertices and $E$ represents the edges. Based on $G$, we can calculate $W \in \mathbb{R}^{n \times n}$ as the adjacency matrix, i.e., $W_{ij} = w(i,j)$.
We let $D \in \mathbb{R}^{n \times n}$ be the corresponding degree matrix, so $D = \text{diag}(d_1, \ldots, d_n)$ where $d_i = \sum_{j \in [n]} w(i,j)$.

For a problem instance of $n$ data points, we define input $X$ as $X=(n, \{z_i\}_{i=1}^n, L, G)$, or $X=(n, L, G)$ if no features are available. We denote the label matrix by $Y \in \{0,1\}^{n \times c}$ where $c$ is the number of classes. Throughout the paper, we assume $c = O_n(1)$, i.e. $c$ is treated as a constant with respect to $n$, which matches most practical scenarios.
Here, $Y_{ij} = 1$ if data point $i \in L$ has label $j \in [c]$ and $Y_{ij} = 0$ otherwise. The goal is to predict the labels of the unlabeled data points. 

An algorithm $F$ in this setting may be considered as a function that takes in $(X, Y)$ and outputs a predictor $f$ that predicts a label in $[c]$ for each data.
We denote $f(z_i)$ as our prediction on the $i$-th data.
To evaluate the performance of a predictor $f$, we use $0$-$1$ loss (i.e. the predictive inaccuracy) defined as $\frac{1}{n}\sum_{i=1}^n \ell_{0-1}\left(f(z_i),y_i\right) = \frac{1}{n}\sum_{i=1}^n \mathbb{I}[f(z_i) \neq y_i].$ In this work, we are interested in the generalizability of an algorithm $F$ on $0$-$1$ loss. 

\paragraph{Hyperparameter Selection.}
We consider several \emph{parameterized families} of classification algorithms. Given a family of algorithms $\mathcal{F}_{\rho}$ parameterized by some parameter $\rho$, and a set of $m$ problem instances $\{(X^{(k)}, Y^{(k)})\}_{k=1}^m$ i.i.d. generated from the data  distribution $\mathcal{D}$ of the input space $\mathcal{X}$ and the label space $\mathcal{Y}$, our goal is to
select a parameter $\hat{\rho}$ whose corresponding prediction function $f_{\hat \rho}$ of algorithm $F_{\hat \rho}$ minimizes the prediction error. That is, denote $f_{\hat \rho}(z_i^{(k)})$ as the predicted label of data point $z_i^{(k)}$ in the $k$-th problem instance, we want
\begin{align*}
\hat{\rho} = \text{argmin}_{\rho} \frac{1}{mn}\sum_{k=1}^m \sum_{i=1}^n \ell_{\mathrm{0-1}}(f_\rho(z_i^{(k)}), y_i^{(k)}).
\end{align*}\label{eqn:algo selection obj}
Each parameter value $\rho$ defines an algorithm $F_\rho$, mapping a problem instance $(X, Y)$ to a prediction function $f_\rho$, which induces a loss $\frac{1}{n}\sum_{i=1}^n \ell_{0-1}(f_\rho(z_i), y_i)$. We define $H_\rho$ as the function mapping $(X, Y)$ to this loss and $\mathcal{H}\rho = {H{\rho'}}_{\rho'}$ as the family of loss functions parameterized by $\rho$.

Note that our problem setting differs from prior theoretical works on graph-based semi-supervised learning. 
The classical setting considers a single algorithm and learning the model parameter from a single problem instance. We are considering \emph{families of algorithms}, each parameterized by a single hyperparameter, and aiming to learn the best \textit{hyperparameter} across \textit{multiple problem instances}.
Our setting combines transductive and inductive aspects: each instance has a fixed graph of size $n$ (transductive), but the graphs themselves are drawn from an unknown meta-distribution (inductive).\looseness-1

\paragraph{Complexity Measures and Generalization Bounds.}
We study the generalization ability of several representative parameterized families of algorithms. That is, we aim to address the question of how many problem instances are required to learn a hyperparameter $\rho$ such that a learning algorithm can perform near-optimally for instances drawn from a fixed problem distribution. Clearly, the more complex the algorithm family, the more number of problem instances are needed.\looseness-1

Specifically, for each algorithm $f_{\hat{\rho}}$ trained given $m$ problem instances, we study the difference in the empirical $0$-$1$ loss and the actual $0$-$1$ on the distribution:
\begin{align*}
\frac{1}{mn}\sum_{k=1}^m &\sum_{i=1}^n \ell_{\mathrm{0-1}}(f_\rho(z_i^{(k)}), y_i^{(k)})\\
- &\mathbb{E}_{\left(X,Y\right) \sim \mathcal{D}}\left[\frac{1}{n}\sum_{i=1}^n \ell_{0-1}\left(f_{\hat \rho}(z_i),y_i\right)\right].
\end{align*}
\label{eqn:algo sec gen}

\noindent To quantify this, we consider two %learning-theoretic 
complexity measures for characterizing the learnability of algorithm families: the \textit{pseudo-dimension} and the \textit{Rademacher complexity}.\looseness-1

\begin{definition}[Pseudo-dimension]
    Let $\mathcal{H}$ be a set of real-valued functions from input space $\mathcal{X}$. We say that $C = (X^{(1)}, ..., X^{(m)}) \in \mathcal{X}^m$ is \textbf{pseudo-shattered} by $\mathcal{H}$ if there exists a vector $r = (r_1, ..., r_m) \in \mathbb{R}^m$ (called ``witness'') such that for all $b = (b_1, ..., b_m) \in \{\pm1\}^m$ there exists $H_b \in \mathcal{H}$ such that $sign(H_b(X^{(k)})-r_k) = b_k$. \textbf{Pseudo-dimension} of $\mathcal{H}$, denoted \textsc{Pdim}($\mathcal{H}$), is the cardinality of the largest set pseudo-shattered by $\mathcal{H}$.
\end{definition}

\noindent The following theorem bounds generalization error using pseudo-dimension.

\begin{theorem}\citep{AnthonyBartlett2009}
\label{thm:pdim gen theorem}
    Suppose $\mathcal{H}$ is a class of real-valued functions with range in $[0, 1]$ and finite $\textsc{Pdim}(\mathcal{H})$. Then for any $\epsilon > 0$ and $\delta \in (0,1)$, for any distribution $\mathcal{D}$ and for any set $S = \{X^{(1)}, \ldots, X^{(m)}\}$ of $m = O\big( \frac{1}{\epsilon^2}\left(\textsc{Pdim}(\mathcal{H}) + \log \frac{1}{\delta}\right)\big)$ samples from $\mathcal{D}$, with probability at least $1-\delta$, we have 
    \begin{align*}
        \left|\frac{1}{m}\sum_{k=1}^mH({X^{(k)}}) - \mathbb{E}_{X \sim \mathcal{D}}[H(X)]\right| \le \epsilon & \text{,  for all $H \in \mathcal{H}$}.
    \end{align*}
\end{theorem}

\noindent Therefore, if we can show $\textsc{Pdim}(\mathcal{H}_\rho)$ is bounded, then using the standard empirical risk minimization argument, Theorem~\ref{thm:pdim gen theorem} implies using $m = O\left(\nicefrac{\textsc{Pdim}(\mathcal{H})}{\epsilon^2}\right)$ problem instances, the expected error on test instances is upper bounded by $\epsilon$. In Section~\ref{sec:label_prop}, we will obtain \emph{optimal} pseudo-dimension bounds for three canonical label-propagation algorithm families.

Another classical complexity measure is the Rademacher complexity:
\begin{definition}[Rademacher Complexity]
    Given a space $\mathcal{X}$ and a distribution $\mathcal{D}$, let $S = \{X^{(1)}, \ldots, X^{(m)}\}$ be a set of examples drawn i.i.d. from $\mathcal{D}$. Let $\mathcal{H}$ be the class of functions $H: \mathcal{X} \rightarrow \mathbb{R}$. The \textbf{(empirical) Rademacher complexity} of $\mathcal{H}$ is 
    \[\hat R_m(\mathcal{H}) = \mathbb{E}_\sigma \left[\sup \left(\frac{1}{m}\sum_{k=1}^m \sigma_k H(X^{(k)})\right)\right],\]
    where each $\sigma_k$ is i.i.d. sampled from $\{-1, 1\}$. 
\end{definition}

\noindent The following theorem bounds generalization error using Rademacher Complexity.
\begin{theorem}\citep{MohriRostamizadehTalwalkar2012}
\label{thm:rc gen}
    Suppose $\mathcal{H}$ is a class of real-valued functions with range in $[0, 1]$. Then for any $\delta \in (0,1)$, any distribution $\mathcal{D}$, and any set $S = \{X^{(k)}\}_{k=1}^m$ of $m$ samples from $\mathcal{D}$, with probability at least $1-\delta$, we have\looseness-1
    \begin{align*}
        &\left|\frac{1}{m}\sum_{k=1}^m H({X^{(k)}}) - \mathbb{E}_{X \sim D}[H(X)]\right| \\
        &\qquad=  O\left( \hat R_m(\mathcal{H}) + \sqrt{\frac{1}{m} \log \frac{1}{\delta}}\right), ~~~\text{for all}~~~ H \in \mathcal{H}.
    \end{align*}
\end{theorem}

\noindent To bound the Rademacher complexity in our setting, we restrict to binary classification $c=2$ and change the label space to $Y \in \{-1, 1\}^{n}$.
For a predictor $f$, we also overload notation and let $f(z_i) \in [0,1]$ be the output probability of node $z_i$ being classified as $1$. Instead of directly using the $0$-$1$ loss, we upper bound it using margin loss, which is defined as\looseness-1
\[\ell_\gamma(f(z_i),y_i) = \1[a_i > 0] + (1 + a_i/\gamma)\1\left[a_i \in \left[-\gamma, 0\right]\right]\]
where  $a_i = -\tau(f(z_i), y_i)= (1-2f(z_i))y_i$. 
Then, $a_i>0$ if and only if $z_i$ is classified incorrectly.

Now we define $H^{\gamma}_{\rho}(X) = \frac{1}{n}\sum_{i=1}^n \ell_\gamma\left(f_{\rho}(z_i),y_i\right)$ to be the margin loss of the entire graph when using a parameterized algorithm $F_\rho$.
Based on this definition, we have an induced loss function family $\mathcal{H}^{\gamma}_\rho$.
Then, given $m$ instances, for any $\gamma > 0$, we can obtain an upper bound for all $H_\rho^\gamma \in \mathcal{H}_\rho^\gamma$:
\begin{align*}
\mathbb{E}&_{\left(X,Y\right) \sim \mathcal{D}}\left[\frac{1}{n}\sum_{i=1}^n \ell_{0-1}\left(f_{\hat{\rho}}(z_i),y_i\right)\right] \\
&\le  \; \mathbb{E}_{\left(X,Y\right) \sim \mathcal{D}}\!\left[\frac{1}{n}\!\sum_{i=1}^n \ell_{\gamma}\left(f_{\hat{\rho}}(z_i),y_i\right)\right] \tag{by definition of $\ell_\gamma$} \\
&= \frac{1}{m} \sum_{i=1}^m H^{\gamma}_{\rho}(X^{(k)}) + O\left(\hat R_m(\mathcal{H}_\rho^\gamma) + \sqrt{\frac{\log{(1/\delta)}}{m}}\right). \tag{by Theorem~\ref{thm:rc gen}}
\end{align*}
Therefore, suppose we find a $\hat{\rho}$ whose empirical margin loss ${1}/{m} \sum_{i=1}^m H^{\gamma}_{\hat{\rho}}(X^{(k)})$ is small, and if we can show $\hat R_m(\mathcal{H}_\rho^\gamma)$ is small, then $F_{\hat{\rho}}$ is a strong algorithm for the new problem instances.
In Section~\ref{sec:gcn}, we bound the Rademacher complexity of graph neural network-based algorithm families.

Note that these guarantees we obtained can also be applied to some standard hyperparameter tuning methods like grid search with cross-validation.
For example, in k-fold cross-validation, if we define the distribution of problem instances as a uniform distribution on these small subsets of validation data, then our results imply the necessary number of iterations $k$ needed to effectively tune hyperparameters using cross-validation. That is, the number of folds of cross-validation needed to learn a hyperparameter that performs nearly as well as the hyperparameter if cross-validation were run to convergence.
\section{Label Propagation-based Families and  Generalization Guarantees}
\label{sec:label_prop}

In this section, we consider three parametric families of label propagation-based algorithms, the classical type of algorithms for semi-supervised learning. 
Label propagation algorithms output a soft-label score $F^* \in \mathbb{R}^{n \times c}$, where the $(i,j)$-th entry of $F^*$ represents the score of class $j$ for the $i$-th sample. The prediction for the $i$-th sample is the class with the highest score, i.e. $\text{argmax}_{j \in [c]} F^*_{ij}$. 

Below we describe each family that we considered and their corresponding pseudo-dimension bounds.
Notably, the bounds for all three families of algorithms are  $\Theta(\log n)$, which implies the existence of efficient algorithms with robust generalization guarantees in this setting.

\subsection{Algorithm Families}
We consider three parametric families described below. 

\paragraph{Local and Global Consistency Algorithm Family ($\mathcal{F}_\alpha$)} 
The first family considered is the local and global consistent algorithms~\citep{zhou2003learning}, parameterized by $\alpha \in (0,1)$.
The optimal scoring matrix $F^*$ is defined as
\[F^*_\alpha = (1- \alpha)(I - \alpha S)^{-1} Y, ~~\text{where }S = D^{-1/2}WD^{-1/2}. \]
Here, $S$ is the symmetrically normalized adjacency matrix. 
This score matrix $F^*_\alpha$ corresponds to minimizing the following objective function
$
    \mathcal{Q}(F) 
    = \frac{1}{2} ( \sum_{i, j=1}^nW_{ij} \| \frac{1}{\sqrt{d_{i}}}F_i - \frac{1}{\sqrt{d_{j}}}F_j\|^2
    + \frac{1-\alpha}{\alpha}\sum_{i=1}^n \|F_i - Y_i\|^2 ).
$
The first term of $Q(F)$ measures the local consistency, i.e., the prediction between nearby points should be similar. The second term measures the global consistency, i.e., consistency to its original label. Therefore, the parameter $\alpha \in (0,1)$ induces a trade-off between the local and the global consistency.
We denote this family as $\mathcal{F}_\alpha$, and the 0-1 losses as $\mathcal{H}_\alpha$.

\paragraph{Smoothing-Based Algorithm Family ($\mathcal{F}_\lambda$)}
This second class of algorithm is parameterized by $\lambda \in (0,+\infty)$~\citep{delalleau2005efficient}. Let $\Delta \in \{0,1\}^{n \times n}$ be a diagonal matrix where elements are 1 only if the index is in the labeled set. The scoring matrix $F^*_\lambda$ is
\[ F^*_\lambda = (S + \lambda I_n \Delta_{i \in L})^{-1} \lambda Y, ~~\text{where} S = D-W. \]
The idea of $\mathcal{F}_\lambda$ is similar to $\mathcal{F}_\alpha$. $\lambda$ is a smoothing parameter that balances the relative importance of the known labels and the structure of the unlabeled points.

\paragraph{Normalized Adjacency Matrix Based Family ($\mathcal{F}_\delta$)}
Here we consider an algorithm family \citep{avrachenkov2012}.
This class of algorithm is parameterized by $\delta \in [0,1]$. The scoring matrix $F^*_\delta$ is 
\[F^*_\delta = (I-c\cdot S)^{-1} Y,
~~\text{where}~~S = D^{-\delta}WD^{\delta - 1}. \] 
Here, $S$ is the (not symmetrically) normalized adjacency matrix and $c \in \mathbb{R}$ is a constant.

This family of algorithms is motivated by $\mathcal{F}_\alpha$ and the family of spectral operators defined in \cite{donnat2023studying}. We may notice that the score matrix $F^*_\delta$ defined here is very similar to $F^*_\alpha$ in the local and global consistency family $\mathcal{F}_\alpha$ when $\alpha$ is set to a constant $c$, whose default value considered in \cite{zhou2003learning} is $0.99$. Here, instead of focusing on the trade-off between local and global consistency, we study the spatial convolutions $S$. With $\delta = 1$, we have the row-normalized adjacency matrix $S = D^{-1}W$. With $\delta = 0$, we have the column-normalized adjacency matrix $S = WD^{-1}$. Finally, with $\delta = 1/2$, we have the symmetrically normalized adjacency matrix that we used in $\mathcal{F}_\alpha$ and many other default implementations~\citep{donnat2023studying, wu19simplifying}.
We denote the set of $0$-$1$ loss functions corresponding to $\mathcal{F}_\delta$ as $\mathcal{H}_\delta$.\looseness-1 

\subsection{Pseudo-dimension Guarantees}

We study the generalization behavior of the three families through pseudo-dimension. The following theorems indicate that all three families have pseudo-dimension $O(\log n)$, where $n$ is the number of data in each problem instance. This result suggests that, all three families of algorithms require $m = O\left(\nicefrac{\log n}{\epsilon^2}\right)$ problem instances to learn a $\epsilon$-optimal algorithmic parameter. We also complement our upper bounds with matching pseudo-dimension lower bound $\Omega(\log n)$, which indicates that we cannot always learn a near-optimal parameter if the number of problem instances is further reduced.

\begin{theorem} \label{thm:alpha upper bound}
    The pseudo-dimension of the Local and Global Consistency Algorithmic Family, $\cF_\alpha$, is $\textsc{Pdim}(\mathcal{H}_\alpha) = \Theta(\log n)$, where $n$ is the total number of labeled and unlabeled data points.
\end{theorem}

\noindent Similar to  Local and Global Consistency, we give a tight $\Theta(\log n)$ bound on the pseudo-dimension of the Smoothing-based family of \citet{delalleau2005efficient}.

\begin{theorem}\label{thm:lambda upper bound}
    The pseudo-dimension of the Smoothing-Based Algorithmic Family, $\cF_\lambda$, is $\textsc{Pdim}(\mathcal{H}_\lambda) = \Theta(\log n)$, where $n$ is the total number of labeled and unlabeled data points.
\end{theorem}

\noindent Finally, and perhaps more surprisingly, we give the same $\Theta(\log n)$ bound on the pseudo-dimension of the normalized adjacency-matrix based family.

\begin{theorem}\label{thm:delta upper bound}
    The pseudo-dimension of the Normalized Adjacency Matrix-Based Algorithmic Family, $\cF_\delta$, is $\textsc{Pdim}(\mathcal{H}_\delta) = \Theta(\log n)$, where $n$ is the total number of labeled and unlabeled data points.
\end{theorem}

\noindent The proofs of the above three theorems follow a similar template. Here, we give an overview of the proof idea. The full proof is in \Cref{appendix:label prop}.

\paragraph{Upper Bound}
First, we investigate the function structure of each index in $F^*$.
For the function classes $\cF_\alpha$ and $\cF_\lambda$, the following lemma is useful.
\begin{lemma}\label{lem:degree of determinant}
     Let $A, B \in \mathbb{R}^{n \times n}$, and $C(x) = (A + xB)^{-1}$ for some $x \in \mathbb{R}$. Each entry of $C(x)$ is a rational polynomial $P_{ij}(x)/Q(x)$ for $i, j \in [n]$ with each $P_{ij}$ of degree at most $n-1$ and $Q$ of degree at most $n$.
\end{lemma}

\noindent This lemma reduces each index in the matrix of form $C(x) = (A+xB)^{-1}$ into a polynomial of parameter $x$ with degree at most $n$. By definition, we can apply this lemma to $F^*_\alpha$ and $F^*_\lambda$ and conclude that each index of these matrices is a degree-$n$ polynomial of variable $\alpha$ and $\lambda$, respectively. 

For the algorithm family $\cF_\delta$, the following lemma is helpful:

\begin{lemma}\label{lem: delta F form}
    Let $S = D^{-x} W D^{x-1} \in \mathbb{R}^{n \times n}$, and $C(x) = (I - c \cdot S)^{-1}$ for some constant $c \in (0,1)$ and variable $x \in [0,1]$. For any $i,j \in [n]$, the $i,j$-the entry of $C(x)$ is an exponential
    $C(x)_{ij} = a_{ij} \exp(b_{ij}x)$
    for some constants $a_{ij}, b_{ij}$.
\end{lemma}

\noindent By definition of $F^*_\delta$, this lemma indicates that each index of $F^*_\delta$ is a weighted sum of $n$ exponentials of the hyperparameter $\delta$.\looseness-1  

For $F^*$ being a prediction matrix of any of the above three algorithmic family, recall that the prediction on each node $i \in [n]$ is $\hat y_i = \text{argmax}_{j \in [c]}([F^*]_{ij})$, so the prediction on a node can change only when $\text{sign}([F^*]_{ij} - [F^*]_{ik})$ changes for some classes $j,k \in [c]$. For the families $\cF_\alpha$ and $\cF_\lambda$, $[F^*]_{ij} - [F^*]_{ik}$ is a rational polynomial $(P_{ij}(\alpha) - P_{ik}(\alpha)) / Q(\alpha)$, where $(P_{ij}(\alpha) - P_{ik}(\alpha))$ and $Q(\alpha)$ are degree of at most $n$ (we can simply replace $\alpha$ with $\lambda$ for $\cF_\lambda$). Therefore, its sign can only change at most $O(n)$ times. For the family $\cF_\delta$, we refer to the following lemma and conclude that the sign of $F^*_{ij}-F^*_{ik}$ can only change at most $O(n)$ times as well.
\begin{lemma} \label{lem:roots of exp sum}
    Let $a_1, \ldots, a_n \in \mathbb{R}$ be not all zero, $b_1, \ldots, b_n \in \mathbb{R}$, and $f(x) = \sum_{i=1}^n a_i e^{b_i x}$. The number of roots of $f$ is at most $n-1$.
\end{lemma}

\noindent Therefore, for all three families, the prediction on a single node can change at most $\binom{c}{2}O(n) \in O(nc^2)$ times as the hyperparameter varies.
For $m$ problem instances, each of $n$ nodes, this implies we have at most $O(mn^2c^2)$ distinct values of the loss function. The pseudo-dimension $m$ then satisfies $2^m \leq O(mn^2c^2)$, which implies $\textsc{Pdim}(\mathcal{H}_{\alpha}) = \textsc{Pdim}(\mathcal{H}_{\lambda}) = \textsc{Pdim}(\mathcal{H}_{\delta}) = O(\log n)$. 

\paragraph{Lower Bound}
Our proof relies on a collection of parameter thresholds and well-designed labeling instances that are shattered by the thresholds. Here we present the proof idea of pseudo-dimension lower bound of the family $\cF_\alpha$. The analysis for $\cF_\lambda$ and $\cF_\delta$ depends on a similar construction. 

We first describe a hard instance of $4$ nodes, using binary labels $a$ and $b$. We have two points labeled $a$ (namely $a_1, a_2$), and one point labeled $b$ (namely $b_1$) connected with both $a_1$ and $a_2$ with edge weight $1$. We also have an unlabeled point $u$ connected to $b_1$ with edge weight $x \geq 0$. That is, the affinity matrix and initial labels are $$W = \begin{bmatrix}
    0 & 1 & 1 & x\\
    1 & 0 & 0 & 0\\
    1 & 0 & 0 & 0\\
    x & 0 & 0 & 0
    \end{bmatrix}, Y = \begin{bmatrix}
    1 & 0 \\
    0 & 1 \\
    0 & 1 \\
    0 & 0 
    \end{bmatrix}.$$
With this construction, the prediction on node $u$ changes and only change when $\alpha = \frac{(x+2)^{1/2}}{2}$. For any $\beta \in [0,1]$ and let $x = 4\beta^2 - 2 \geq 0$, then $\hat y_4 = 0$ when $\alpha < \beta$ and $\hat y_4 = 1$  when $\alpha \geq \beta$.

Now we can create a large graph of $n$ nodes, consisting of $n/4$ connected components as described above. We assume $4$ divides $n$ for simplicity. Given a sequence of $\alpha$'s such that $0 < \alpha_0 < 1/\sqrt{2} \leq  \alpha_1 < \alpha_2 < ... < \alpha_{n/4} < 1$, we can create the $i$-th connected component with $x = 4\alpha_i^2 -2$. Now the predicted label of the unlabeled node in the $i$-th connected component is $0$ when $\alpha < \alpha_i$ and $1$ when $\alpha \geq \alpha_i$. By alternatively labeling these unlabeled nodes, the $0$-$1$ loss of this problem instance fluctuates as $\alpha$ increases.
    
Finally, by precisely choosing the subsequences so that the oscillations align with the bit flips in the binary digit sequence, we can construct m instances that satisfy the $2^m$ shattering constraints.
\begin{figure}[ht]
    \centering
    \includegraphics[width=0.5\textwidth]{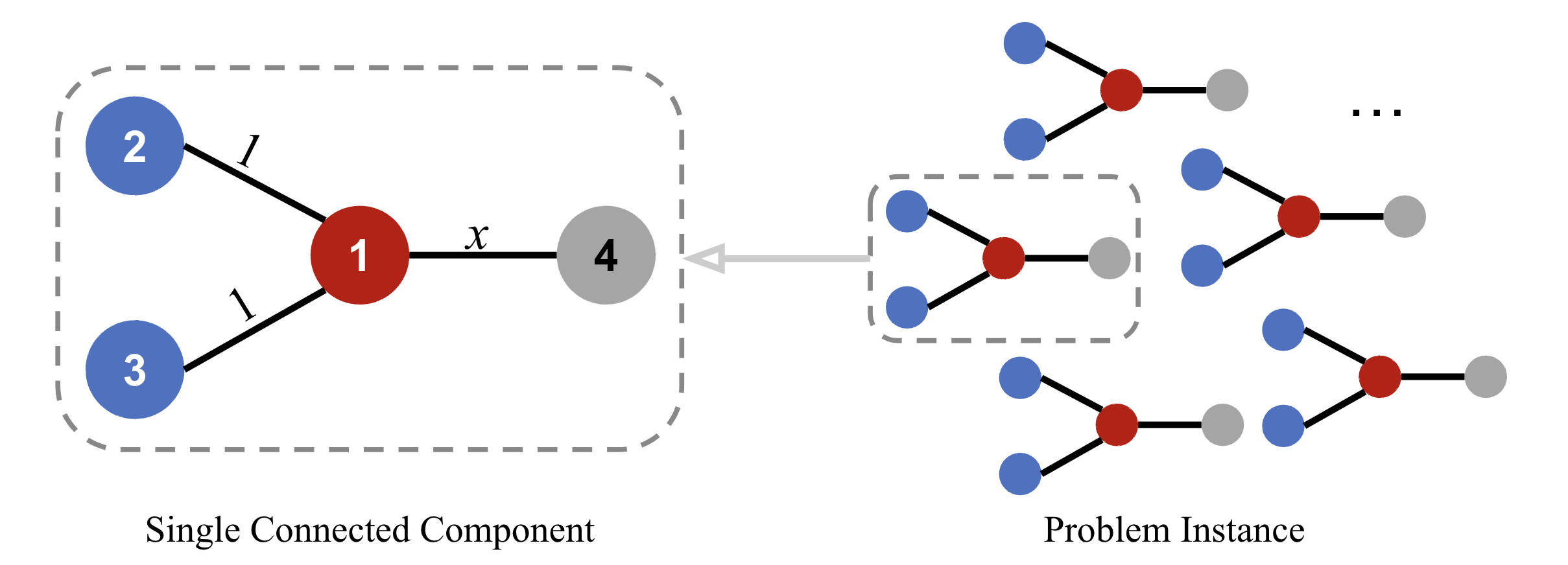}
    \caption{An illustration of the construction of the problem instance in the lower bound proof. 
    }
    \label{fig:instance_graph}
\end{figure}
\begin{remark}
    We reiterate the implications of the above three theorems. All three families have pseudo-dimension $\Theta(\log n)$. This indicates that all three families of algorithms require $m = O\left(\nicefrac{\log n}{\epsilon^2}\right)$ problem instances to learn an $\epsilon$-optimal hyperparameter. 
\end{remark}

\section{GNN Families and  Generalization Guarantees}\label{sec:gcn}
In this section, we study hyperparameter selection for Graph Neural Networks (GNNs)~\citep{kipf2016semi,velivckovic2017graph,iscen2019label}, which excel in tasks involving graph-structured data like social networks, recommendation systems, and citation networks. To understand generalization in hyperparameter selection for GNNs, we analyze Rademacher complexity. 

To the best of our knowledge, we are the first to provide generalization guarantees for hyperparameter selection. Prior work~\citep{garg2020generaliz} focused on Rademacher complexity for graph classification with fixed hyperparameters, whereas we address node classification across multiple instances, optimizing hyperparameters.

In Section~\ref{sec:sgc}, we examine the Rademacher complexity bound of a basic Simplified Graph Convolutional Network~\citep{wu19simplifying} family, as a foundation for the more complex family. 

In Section~\ref{sec:interpolated GCAN}, we introduce a novel architecture, which we call GCAN, that uses a hyperparameter $\eta \in [0,1]$ to interpolate two popular GNNs: the graph convolutional neural networks (GCN) and graph attention neural networks (GAT). GCAN selects the optimal model for specific datasets: $\eta = 0$ corresponds to GCN, $\eta = 1$ to GAT, and intermediate values explore hybrid architectures that may outperform both. We also establish a Rademacher complexity bound for the GCAN family.

Our proofs for SGC and GCAN share a common strategy: modeling the $0$-$1$ loss of each problem instance as an aggregation of single-node losses, reducing the problem to bounding the Rademacher complexity of computation trees for individual nodes. Specifically, we upper bound the $0$-$1$ loss with a margin loss, then relate the complexity of problem instances to the computation trees of nodes. Using a covering argument, we bound the complexity of these trees by analyzing margin loss changes due to parameter variations.\looseness-1

For each node $z_i$, we define its computation tree of depth $L$ to represent the structured $L$-hop neighborhood of $v$, where the children of any node $u$ are the neighbors of $u$, $\mathcal{N}_u$. Denote the computation tree of $z_i$ as $t_i$, and the learned parameter as $\theta$, then $l_\gamma(z_i) = l_\gamma(t_i, \theta)$. 
We can now rewrite $l_\gamma(Z)$ as an expectation over functions applied to computation trees. Let $t_1, ..., t_t$ be the set of all possible computation trees of depth $L$, and $w_i(Z)$ the number of times $t_i$ occurs in $Z$. Then, we have
\[l_\gamma(Z) = \sum_{i=1}^t \frac{w_i(Z)}{\sum_{j=1}^t w_j(Z)} l_\gamma(t_i, \theta) = \mathbb{E}_{t\sim w^\prime(Z)} l_\gamma(t,\theta).\]

\noindent The following proposition indicates that it suffices to bound the Rademacher Complexity of single-node computation trees.\looseness-1 

\begin{proposition}[Proposition 6 from \cite{garg2020generaliz}.]
    Let $S = \{Z_1, ..., Z_m\}$ be a set of i.i.d. graphs, and let $\mathcal{T} = \{t_1, ..., t_m\}$ be such that $t_j \sim w^\prime (Z_j), j \in [m]$. Denote by $\hat R_{S}$ and $\hat R_{\mathcal{T}}$ the empirical Rademacher complexity of $\mathcal{H}_\rho^\gamma$s for graphs $S$ and trees $\mathcal{T}$. Then, $\hat R_{S} = \mathbb{E}_{t_1, ..., t_m} \hat R_{\mathcal{T}}$.
\end{proposition}

\subsection{Simplified Graph Convolutional Network Family}\label{sec:sgc}

Simplified Graph Convolution Network (SGC) is introduced by \citet{wu19simplifying}. By removing nonlinearities and collapsing weight matrices between consecutive layers, SGC reduces the complexity of GCN while maintaining high accuracy in many applications.

Consider input data $X = (n, Z, L, G)$, where the feature is written as a matrix $Z \in \R^{n \times d}$. For any value of the hyperparameter $\beta \in [0,1]$, let $\tilde{W} = W + \beta I$ be the augmented adjacency matrix, $\tilde{D} = D + \beta I$ be the corresponding degree matrix, and $S = \tilde{D}^{-1/2}\tilde{W}\tilde{D}^{-1/2}$ be the normalized adjacency matrix. Let $\theta \in \mathbb{R}^{d}$ be the learned parameter. The SGC classifier of depth $L$ is
\[\hat Y = \text{softmax}(S^L Z \theta).\]

\noindent We focus on learning the algorithm hyperparameter $\beta \in [0,1]$ and define the SGC algorithm family as $\mathcal{F}_{\beta}$. We denote the class of margin losses induced by $\cF_\beta$ as $\mathcal{H}_{\beta}^{\gamma}$.
To study the generalization ability to tune $\beta$, we bound the Rademacher complexity of $\mathcal{H}_{\beta}^{\gamma}$. The proof is detailed in Appendix~\ref{appendix:SGC}.

\begin{theorem}\label{thm:rc of SGC}
    Assuming $D, W,$ and $Z$ are bounded (the assumptions in \cite{bartlett2017spectrally,garg2020generaliz}), i.e. $d_i \in [C_{dl}, C_{dh}] \subset \mathbb{R}^+$, $w_{ij}\in [0,C_w]$, and $\|Z\|\leq C_z$, we have that the 
    Rademacher complexity of $\mathcal{H}^{\gamma}_{\beta}$ is bounded:
    \begin{align*}
    \hat R_{m}(\mathcal{H}^{\gamma}_{\beta})
    = O\left( \frac{\sqrt{dL\log\frac{C_{dh}}{C_{dl}} + d\log{\frac{mC_zC_\theta}{\gamma}}}}{\sqrt{m}}\right).
    \end{align*}
\end{theorem}
\noindent This theorem indicates that the number of problem instances needed to learn a near-optimal hyperparameter only scale polynomially with the input feature dimension $d$ and the number of layers $L$ of the neural networks, and only scales logarithmically with the norm bounds $C$'s and the margin $\gamma$. 

\subsection{GCAN Interpolation and  Rademacher Complexity Bounds} \label{sec:interpolated GCAN}
In practice, GCN and GAT outperform each other in different problem instances \citep{dwivedi2023benchmarking}. To effectively choose the better algorithm, we introduce a family of algorithms that \emph{interpolates} GCN and GAT, parameterized by $\eta \in [0,1]$. 
This family includes both GCN and GAT, so by choosing the best algorithm within this family, we can automatically select the better algorithm of the two, specifically for each input data. 
Moreover, GCAN could potentially outperform both GAT and GCN by taking $\eta$ as values other than $0$ and $1$.
We believe such an interpolation technique could potentially be used to select between other algorithms that share similar architecture. 

Recall that in both GAT and GCN, the update equation has the form of activation and a summation over the feature of all neighboring vertices in the graph (a brief description of GAT and GCN is given in \Cref{appendix:GAT_GCN}). Thus, we can interpolate between the two update rules by introducing a hyperparameter $\eta \in [0,1]$, where $\eta = 0$ corresponds to GCN and $\eta=1$ corresponds to GAT. Formally, given input $X = (n, \{z_i\}_{i=1}^n, L, G)$, we initialize $h_i^0 = z_i$ and update at a level $\ell$ by
\begin{align*}
    &h_i^{\ell} = \sigma\left(\sum_{j \in \mathcal{N}_i} \left(\eta \cdot e_{ij}^{\ell} + (1 - \eta) \cdot \frac{1}{\sqrt{d_i d_j }} \right)  U^{\ell}  h_j^{\ell} \right),
\end{align*}
{where}
\begin{align*}
    &e_{ij}^{\ell} = \frac{\exp(\hat{e}_{ij}^{\ell})}{\sum_{j' \in \mathcal{N}_i} \exp(\hat{e}_{ij'}^{\ell})}, ~ \hat{e}_{ij}^{\ell} = \sigma(V^{\ell} [U^{\ell}h_i^{\ell}, U^{\ell}h_j^{\ell}]).
\end{align*}
Here $e_{ij}^{\ell}$ is the attention score of node $j$ for node $i$. $V^{\ell}$ and $U^{\ell}$ are learnable parameters. $\sigma(\cdot)$ is a $1$-Lipschitz activation function (e.g. ReLU, sigmoid, etc.).
$[U^{\ell}h_i^{\ell}, U^{\ell}h_j^{\ell}]$ is the concatenation of $U^{\ell}h_i^{\ell}$ and $U^{\ell}h_j^{\ell}$. 
We denote this algorithm family by $\mathcal{F}_\eta$ and the induced margin loss class by $\mathcal{H}_\eta^\gamma$. 

While our primary focus is not the comparative performance of GCAN against GAT or GCN, our curiosity led us to conduct additional experiments, presented in \Cref{appendix:experiments}. The results consistently show that GCAN matches or exceeds the performance of both GAT and GCN.

\begin{theorem}\label{thm:rc of GCAN}
    Assume the parameter $U^\ell$ is shared over all layers, i.e. $U^\ell = U$ for all $\ell \in [L]$ (the assumption used in~\cite{garg2020generaliz}), and the parameters are bounded: $\|U\|_F \leq C_U$, $\|V^\ell\|_2 \leq C_V$, $\|z_i\|\leq C_z$, and $d_i \in [C_{dl}, C_{dh}]$. Denoting the branching factor by $r = \max_{i \in [n]} |\sum_{j \in [n]}\ind[w_{ij} \neq 0]|$, we have that 
    the Rademacher complexity of $\mathcal{H}^\gamma_{\eta}$ is bounded:
    \begin{align*}
        \hat R_{m}(\mathcal{H}^\gamma_{\eta})
        = O\left( \frac{d\sqrt{L\log \frac{rC_U}{C_{dl} + C_U} + \log\frac{mdC_z}{\gamma}}}{\sqrt{m}} \right).
    \end{align*}
\end{theorem}
\noindent The proof of Theorem~\ref{thm:rc of GCAN} is similar to that of Theorem~\ref{thm:rc of SGC}. See Appendix~\ref{appendix:GCAN} for details.
%%% putting it here to resolve formatting issues
\begin{figure*}[ht]
    \centering
    \includegraphics[width=0.99\linewidth]{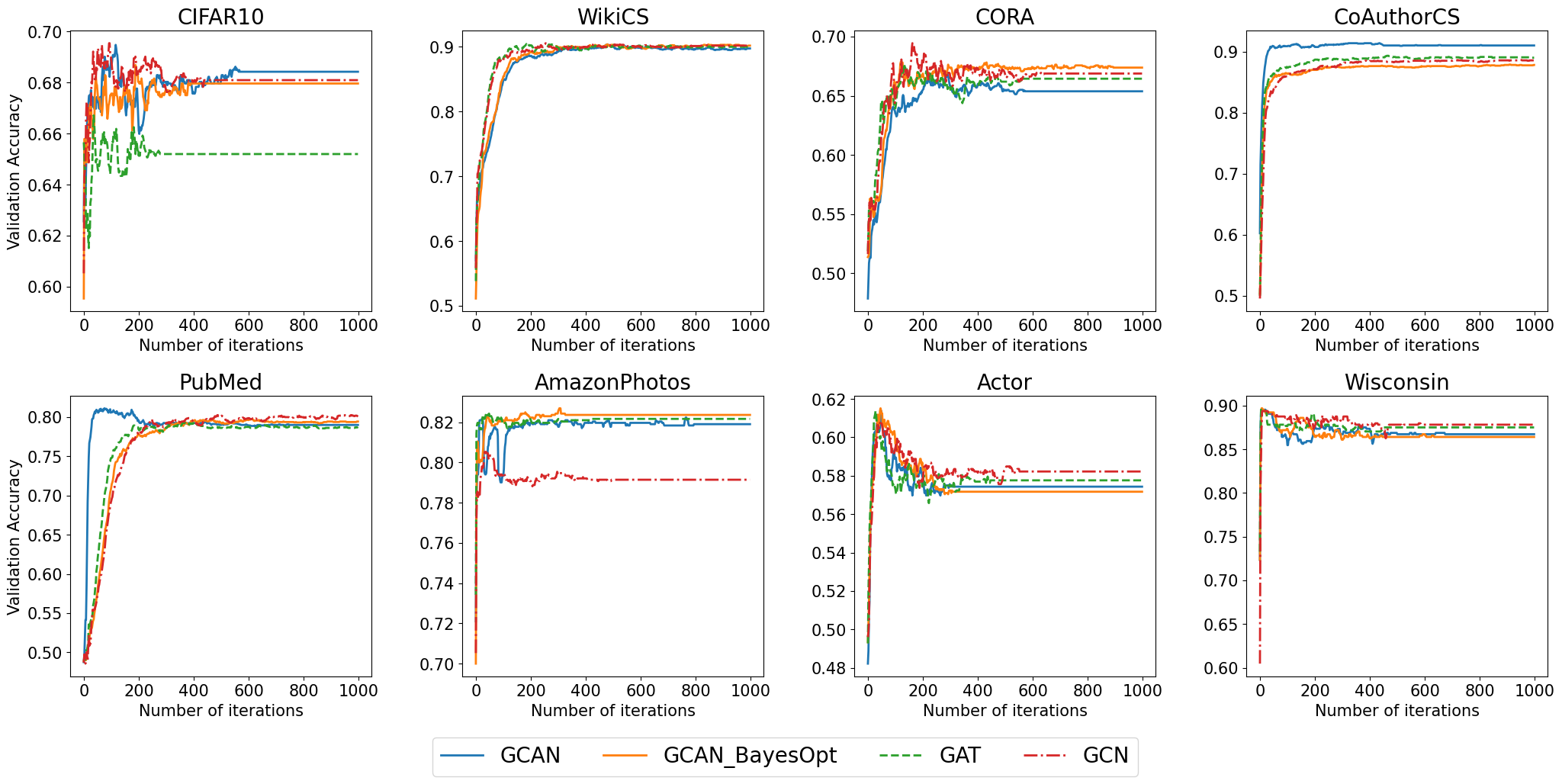}
    \caption{Validation Accuracy (computed on the unlabeled nodes across 20 testing graphs) vs. iterations. GCAN competes with the better accuracy between GAT and GCN across datasets. }
    \label{fig:gcan_multi}
\end{figure*}
%%%
\begin{remark}
    The main difference between the Rademacher Complexity of Simplified Graph Convolution Network (\Cref{thm:rc of SGC}) and GCAN (\Cref{thm:rc of GCAN}) is the dependency on feature dimension $d$: $\sqrt{d}$ for SGC and $d$ for GCAN. This difference arises from the dimensionality of the parameters. The parameter $\theta$ in SGC has dimension $d$, but the parameter $U$ and $V$ in GCAN have dimension $d \times d$ and $1 \times 2d$, respectively. As GCAN is a richer model, it requires more samples to learn, but this is not a drawback; its complexity allows it to outperform SGC in many scenarios. 
\end{remark}
\begin{remark}
    There are no direct dependencies on $n$ in \Cref{thm:rc of SGC} and \Cref{thm:rc of GCAN}, but the dependency is implicitly captured by the more fine-grained value $C_{dl}, C_{dh}$, and $C_Z$. Here, $C_{dl}$ and $C_{dh}$ are the lower and upper bounds of the degree (number of neighbors) of the nodes, which generally increase with $n$.
    $C_Z$ is the Frobenius norm of the feature matrix $Z \in \mathbb{R}^{n \times d}$. Since the size of $Z$ scales with $n$, the value of $C_Z$ is generally larger for larger $n$.
\end{remark}
\section{Experiments}
In this section, we empirically verify the effectiveness of our hyperparameter selection method.

We focus on our GCAN architecture, aiming to demonstrate our approach's effectiveness for selecting algorithm hyperparameters in our setup. To illustrate this, we compare the performance of GCAN with tuned hyperparameters against GAT and GCN.

For each dataset, we sample 20 random sub-graphs of 100 nodes to learn the optimal hyperparameter $\eta$ via backpropagation. A large disconnected graph is formed by combining these sub-graphs, allowing parameter values to vary across graphs while sharing a unified learnable $\eta$. The optimized hyperparameter is then tested on another 20 test sub-graphs from the same dataset.

We also compare our backpropagation-based approach with Bayesian Optimization (see e.g. \cite{frazier2018tutorialbayesianoptimization}). Using the same 20 training sub-graphs, we perform Bayesian Optimization to select the hyperparameter $\eta$, ensuring both methods use an equal number of forward passes. The selected $\eta$ is then evaluated on a separate set of 20 test sub-graphs from the same dataset.

The results on the test set are shown in \Cref{fig:gcan_multi}. Note that GCN outperforms GAT on some datasets (e.g.\ CORA, CoAuthorCS) and GCN performs better on others (e.g.\ CIFAR10, see also \cite{dwivedi2023benchmarking}). With GCAN, we can achieve the best performance on most datasets.
Indeed, as seen in \Cref{fig:gcan_multi}, GCAN consistently achieves higher or comparable accuracy compared to both GAT and GCN across all datasets. Notably, GCAN demonstrates significant improvements in CIFAR10 and CoAuthorCS, highlighting its effectiveness in these scenarios. Also, comparing backpropagation with Bayesian Optimization, backpropagation achieves better performance on more datasets (e.g. \ CIFAR10, CoAuthorCS, Actor), but Bayesian Optimization is more effective in certain datasets (e.g. \ CORA, AmazonPhotos).

In \Cref{appendix:experiments}, we also conduct experiments to empirically verify the results in \Cref{sec:label_prop}. We show that by selecting the number of problem instances $m = O(\log n / \epsilon^2)$, the empirical generalization error is within $O(\epsilon)$, matching our theoretical results. We also have further details on the empirical setup and the variation of the accuracy of GCAN with the hyperparameter $\eta$ in the Appendix. 
\section{Conclusion}
We study the problem of hyperparameter tuning in graph-based semi-supervised learning for both classical label-propagation based techniques as well as modern deep learning algorithms. For the former, we obtain tight learning guarantees by bounding the pseudo-dimension of the relevant loss function classes. For the latter, we study a novel interpolation of convolutional and attention based graph neural network architectures and provide data-dependent bounds on the complexity of tuning the  hyperparameter the interpolates the two architectures. We obtain a sharper generalization error bound for tuning the hyperparameter in the simplified graph convolutional networks proposed in prior work. Our experiments indicate that we can achieve consistently good empirical performance across datasets by tuning the interpolation parameter.

An interesting direction for further investigation involves improving computational efficiency. \citet{sharma2023efficiently} introduced techniques for approximating loss functions to reduce the cost of tuning graph kernel hyperparameters. It would be worthwhile to explore whether their methods can be adapted to our setting to alleviate computational burdens.

Another natural extension is the tuning of multiple hyperparameters. Although this increases analytical complexity, we anticipate that our techniques remain applicable. For GNN-based algorithms, Rademacher Complexity may still provide a suitable capacity measure. Our approach would aim to bound the variation in predicted values as hyperparameters change and then apply a covering argument. For label propagation methods, we would instead analyze how the scoring matrix evolves with respect to hyperparameter shifts. These directions offer a compelling foundation for extending our techniques to more complex tuning scenarios.

\section*{Acknowledgments}

We thank Nina Balcan and anonymous reviewers for helpful comments. This work was supported in part by the National Science Foundation under grants ECCS-2216899 and ECCS-2216970.

\bibliographystyle{plainnat}
\bibliography{references}

%%%%%%%%%%%%%%%%%%%%%%%%%%%%%%%%%%%%%%%%%%%%%%%%%%%%%%%%%%%%%%%%%%%%%%%%%%%%%%%
%%%%%%%%%%%%%%%%%%%%%%%%%%%%%%%%%%%%%%%%%%%%%%%%%%%%%%%%%%%%%%%%%%%%%%%%%%%%%%%
% APPENDIX
%%%%%%%%%%%%%%%%%%%%%%%%%%%%%%%%%%%%%%%%%%%%%%%%%%%%%%%%%%%%%%%%%%%%%%%%%%%%%%%
%%%%%%%%%%%%%%%%%%%%%%%%%%%%%%%%%%%%%%%%%%%%%%%%%%%%%%%%%%%%%%%%%%%%%%%%%%%%%%%
\newpage
\appendix
\onecolumn
\section*{Appendix}
\section{Proofs in Section~\ref{sec:label_prop}} \label{appendix:label prop}
We provide additional proof details from Section~\ref{sec:label_prop} below. 
\subsection{Proof of Lemma~\ref{lem:degree of determinant}}\label{appendix:proof of determinant}

\begin{proof}
    Using the adjugate matrix, we have
    $$C(x) = \frac{1}{\det(A+xB)}\text{adj}(A+xB).$$
    The determinant of $A+xB$ can be written as
    $$\det(A+xB) = \sum_{\sigma \in S_n}\left(\text{sgn}(\sigma) \prod_{i=1}^n [A+xB]_{i\sigma_i}\right),$$
    where $S_n$ represents the symmetric group and $\text{sgn}(\sigma) \in \{\pm 1\}$ is the signature of permutation $\sigma$. Thus $\det(A+xB)$ is a polynomial of $x$ with a degree at most $n$.  The adjugate of $A+xB$ is  $$\text{adj}(A+xB) = C^\top,$$ where $C$ is the cofactor matrix of $A+xB$. By definition, each entry of $C$ is $C_{ij} = (-1)^{i+j}k_{ij}$ where $k_{ij}$ is the determinant of the $(n-1) \times (n-1)$ matrix that results from deleting $i$-th row and $j$-th column of $A+xB$. This implies that each entry of $C$ (and thus $\text{adj}(A+xB)$) is a polynomial of degree at most $n-1$. Letting $Q(x) = \det (A+xB)$ and $P_{ij}(x) = [\text{adj}(A+xB)]_{ij}$ concludes our proof.
\end{proof}

\subsection{Proof of Lemma~\ref{lem: delta F form}}
\begin{proof}
    The $ij$-th element of $I - c \cdot S$ is 
\[[I - c \cdot S]_{ij} = \begin{cases}
    -c \cdot d_{i}^{-\delta}W_{ij}d_{j}^{\delta-1} = -(d_{i}^{-1}d_{j})^{\delta}(c \cdot W_{ij}d_{j}^{-1}) & \text{, if $i \neq j$}\\
    1=(d_{i}^{-1}d_{i})^\delta & \text{, otherwise.}
\end{cases}\]
Using adjugate matrix, we have
    $$(I-c\cdot S)^{-1} = \frac{1}{\det(I-c \cdot S)}\text{adj}(I-c\cdot S).$$
Note that the determinant of any $k \times k$ matrix $A$ can be written as 
$$\det(A) = \sum_{\sigma \in S_k}\left(\text{sgn}(\sigma) \prod_{i=1}^k [A]_{i\sigma_i}\right),$$
where $S_k$ represents the symmetric group and $\text{sgn}(\sigma) \in \{\pm 1\}$ is the signature of permutation $\sigma$. 

Now consider $\text{adj}(I-c \cdot S)$. Let $M_{ij}$ be the $(n-1) \times (n-1)$ matrix resulting from deleting $i$-th row and $j$-th column from $[I- c \cdot S]$. Then, 
$$[\text{adj}(I - c \cdot S)]_{ij} = (-1)^{i+j}\det(M_{ji}) = \sum_{\sigma \in S_{n-1}}\left(\text{sgn}(\sigma) \prod_{k=1}^{n-1} [M_{ji}]_{k\sigma_k}\right) = \sum_{\sigma \in S_{n-1}}\left(a_\sigma \exp(\delta\ln b_\sigma)\right),$$
for some constants $a_\sigma, b_\sigma$ that satisfies

$$b_\sigma = (\prod_{k \in [n]\backslash \{j\}} d_{k}^{-1})(\prod_{k \in [n]\backslash\{i\}}d_{k}) = d_{i}^{-1}d_{j}.$$

We can then rewrite $[\text{adj}(I- c \cdot S)]_{ij}$ as
$$[\text{adj}(I- c \cdot S)]_{ij} = \sum_{\sigma \in S_{n-1}} (a_\sigma \exp(\delta \ln(d_{i}^{-1}d_{j}))) = a_{ij}\exp(\delta \ln(d_{i}^{-1}d_{j})),$$
where $a_{ij} = \sum_{\sigma \in S_{n-1}} a_\sigma$.
\end{proof}

\subsection{Proof of Lemma~\ref{lem:roots of exp sum}}
\begin{proof}
We prove by induction on $n$. If $n=1$, then $f(x) = ae^{bx}$ and $a \neq 0$, so $f(x)$ has $0=n-1$ root. Now assume that the statement holds for some $n=m$ and consider when $n=m+1$. That is, we have
$$f(x) = \sum_{i=1}^{m+1} a_i e^{b_i x}.$$
Assume for the sake of contradiction that $f$ has $n = m+1$ roots. Define  
$$g(x) = \frac{f(x)}{e^{b_{m+1}x}} = \sum_{i=1}^m a_ie^{(b_i - b_{m+1})x} + a_{m+1},$$
then $g$ also has $m+1$ roots. Since $g$ is continuous, 
$$g^\prime (x) = \sum_{i=1}^m (b_i - b_{m+1})a_ie^{(b_i - b_{m+1})x}$$
must have $m$ roots. However, using our induction hypothesis, it should have at most $m-1$ roots. This means our assumption is incorrect, i.e. $f$ must have at most $m = n-1$ roots. 

We conclude that $f$ must have at most $n-1$ roots. 
\end{proof}

\subsection{Proof of Theorem~\ref{thm:alpha upper bound}}\label{appendix:alpha}
\paragraph{Upper Bound.} Proof is given in \Cref{sec:label_prop}.

\paragraph{Lower Bound.} We first construct the small connected component of $4$ nodes:
\begin{lemma} \label{lem:alpha connected components}
Given $x\in [1/\sqrt{2},1)$, there exists a labeling instance $(G, L)$ with $4$ nodes, such that the predicted label of the unlabeled points changes only at $\alpha = x$ as $\alpha$ varies in $(0,1)$.
\end{lemma}
\begin{proof}
We use binary labeling $a$ and $b$. We have two points labeled $a$ (namely $a_1, a_2$), and one point labeled $b$ (namely $b_1$) connected with both $a_1$ and $a_2$ with edge weight $1$. We also have an unlabeled point $u$ connected to $b_1$ with edge weight $x \geq 0$. That is, the affinity matrix and initial labels are $$W = \begin{bmatrix}
    0 & 1 & 1 & x\\
    1 & 0 & 0 & 0\\
    1 & 0 & 0 & 0\\
    x & 0 & 0 & 0
\end{bmatrix}, Y = \begin{bmatrix}
    1 & 0 \\
    0 & 1 \\
    0 & 1 \\
    0 & 0 
\end{bmatrix}.$$
Recall that the score matrix is $$F^* = (1-\alpha)(I -\alpha S)^{-1}Y .$$ 
We now calculate: 
\begin{align*}
    D^{-1/2} =& \begin{bmatrix}
    (x+2)^{-1/2} & 0 & 0 & 0\\
    0 & 1 & 0 & 0\\
    0 & 0 & 1 & 0\\
    0 & 0 & 0 & x^{-1/2}
\end{bmatrix},
\end{align*}
\begin{align*}
S = D^{-1/2}WD^{-1/2} =&
\begin{bmatrix}
0 & (x+2)^{-1/2} & (x+2)^{-1/2} & x^{1/2}(x+2)^{-1/2} \\
(x+2)^{-1/2} & 0 & 0 & 0\\
(x+2)^{-1/2} & 0 & 0 & 0\\
x^{1/2}(x+2)^{-1/2} & 0 & 0 & 0
\end{bmatrix},
\end{align*}
\begin{align*}
(I-\alpha S)^{-1} =& 
\frac{1}{\det(I-\alpha S)}\text{adj}(I - \alpha S)\\
=& \frac{1}{1-\alpha^2}\begin{bmatrix}
    1 & \frac{\alpha}{(x+2)^{1/2}}  & \frac{\alpha}{(x+2)^{1/2}}  & \frac{\alpha x^{1/2}} {(x+2)^{1/2}} \\
    \frac{\alpha}{(x+2)^{1/2}} & 1-\frac{\alpha^2(x+1)x}{(x+2)} & \frac{\alpha^2}{x+2} & \frac{\alpha^2 x^{1/2}}{(x+2)}
    \\
    \frac{\alpha}{(x+2)^{1/2}} & \frac{\alpha^2}{x+2} &  1-\frac{\alpha^2(x+1)x}{(x+2)} & \frac{\alpha^2 x^{1/2}}{(x+2)}
    \\
    \frac{\alpha x^{1/2}} {(x+2)^{1/2}} & \frac{\alpha^2 x^{1/2}}{(x+2)} & \frac{\alpha^2 x^{1/2}}{(x+2)} & 1-\frac{2\alpha^2}{x+2}
\end{bmatrix}.
\end{align*}
Recall that the prediction on the unlabeled point is $\hat y_4 = \text{argmax}F^*_4$, so we calculate 
\begin{align*}
\hat y_4 = \text{sign} (F^*_{4, 2} - F^*_{4, 1}) 
=& \text{sign}\left(
\frac{\alpha x ^{1/2}(2\alpha - (x+2)^{1/2})}
{(1+\alpha)(x+2) }\right)\\
=& \text{sign}\left(x ^{1/2}(2\alpha - (x+2)^{1/2})\right). \tag{since $\alpha \in (0,1)$ and $x \geq 0$}
\end{align*}
Solving the equation $x ^{1/2}(2\alpha - (x+2)^{1/2}) = 0$, we know that the prediction changes and only change when $\alpha = \frac{(x+2)^{1/2}}{2}$. Let $x = 4x^2 - 2 \geq 0$, then $\hat y_4 = 0$ when $\alpha < x$ and $\hat y_4 = 1$  when $\alpha \geq x$, which completes our proof. 
\end{proof}

\begin{lemma}\label{lem:alpha alternating sign}
Given integer $n > 1$ and a sequence of $\alpha$'s such that $0 < \alpha_0 < 1/\sqrt{2} \leq  \alpha_1 < \alpha_2 < ... < \alpha_n < 1$, there exists a real-valued witness $w>0$ and a problem instance of partially labeled $4n$ points, such that for $0 \leq i \leq n/2-1$, $l<w$ for $\alpha \in (\alpha_{2i}, \alpha_{2i+1})$, and $l>w$ for $\alpha \in (\alpha_{2i+1}, \alpha_{2i+2})$.
\end{lemma}
\begin{proof}
We create $n$ connected components using the previous lemma, with $x_i = \alpha_i$. Let the unlabeled point in the $i$th component be $u_i$, then as $\alpha$ increases from $\alpha_{i-1}$ to $\alpha_i$, the predicted label of $u_i$ changes from $a$ to $b$. If the sequence $u_i$ is alternately labeled with $u_1$ labeled $a$, then the loss increases and decreases alternately as all the labels turn to $b$ when $\alpha$ increases to $\alpha_n$. Specifically, as $\alpha$ increases to $\alpha_1$, the point $u_1$ has predicted label changes from $a$ to $b$. Since its true label is $a$ and the predicted labels of other $u_i$'s remain unchanged, our loss slightly increases to $l_{max}$. Then, as $\alpha$ increases to $\alpha_2$, the point $u_2$ gets correctly labeled as $b$ and all other nodes unchanged, which slightly decreases our loss back to $l_{min}$. The loss thus fluctuates between $l_{min}$ and $l_{max}$. We therefore set the witness $w$ as something in between. 
$$w = 
\frac{l_{min}+l_{max}}{2}.$$
\end{proof}

We now finish the lower bound proof for Theorem~\ref{thm:alpha upper bound}.

\begin{proof}
Arbitrarily choose $n^\prime = n/4$ (assumed to be a power of 2 for convenient representation) real numbers $1/\sqrt{2} \leq \alpha_{[000..1]} < \alpha_{[000...10]} < ...< \alpha_{[111...11]} < 1$. The indices are increasing binary numbers of length $m = \log n^\prime$. We create $m$ labeling instances that can be shattered by these $\alpha$ values. For the $i$-th instance $(X^{(i)}, Y^{(i)})$, we apply the previous lemma with a subset of the $\alpha_b$ sequence that corresponds to the $i$-th bit flip in $b$, where $b \in \{0,1\}^m$. For example, $(X^{(1)}, Y^{(1)})$ is constructed using $r_{[100..0]}$, and $(X^{(2)}, Y^{(2)})$ is constructed using $r_{[010..0]}, r_{[100.0]}$ and $r_{[110..0]}$. The lemma gives us both the instances and the sequence of witnesses $w_i$.  

This construction ensures $\text{sign}(l_{\alpha_b}-w_i) = b_i$ for all $b \in \{0,1\}^m$. Thus the pseudo-dimension is at least $\log n^\prime = \log n - \log 4 = \Omega(\log n)$
\end{proof}

\subsection{Proof of Theorem~\ref{thm:lambda upper bound}}\label{appendix:lambda upper bound}
\paragraph{Upper Bound.}
    The closed-form solution $F^*$ is given by 
    $$ F^*= (S + \lambda I_n \Delta_{i \in L})^{-1} \lambda Y.$$
    By Lemma~\ref{lem:degree of determinant}, each coefficient $[F^*]_{ij}$ is a rational polynomial in $\lambda$ of the form $P_{ij}(\lambda) / Q(\lambda) $ where $ P_{ij}$ and $Q$ are polynomials of degree $n$ and $n$ respectively. Note that the prediction for each node $ i \in [n]$ is $\hat{y}_i = \argmax_{j\in c} f_{ij} $ and thus the prediction on any node in the graph can only change when $sign(f_{ij} - f_{ik}) $ changes for some $j, k \in [c]$. Note that $f_{ij} - f_{ik}$ is also a rational polynomial $(P_{ij}(\lambda) - P_{ik}(\lambda))/Q(\lambda) $ where both the numerator and denominator are polynomials in $\lambda$ of degree $n$, meaning the sign can change at most $O(n) $ times. As we vary $\lambda$, we have that the prediction on a single node can change at most $\binom c2 O(n) \in O(n c^2)$. Across the $m$ problem instances and the $n$ total nodes, we have at most $O(n^2c^2m) $ distinct values of our loss function. The pseudo-dimension $m$ thus satisfies $2^m \leq O(n^2c^2m)$, or $m = O(\log n)$  

\paragraph{Lower Bound.} We construct the small connected component of $4$ nodes as follows:
\begin{lemma}\label{lem:lambda connected components}
Given $\lambda' \in (1, \infty)$, there exists a labeling instance $(X, Y)$ with $4$ nodes, such that the predicted label of the unlabeled points changes only at $\lambda = \lambda'$ as $\lambda$ varies in $(0,\infty)$.
\end{lemma}
\begin{proof}
We use binary labeling $a$ and $b$.  We have two points labeled $a$ (namely $a_1, a_2$), and one point labeled $b$ (namely $b_1$). We also have an unlabeled point $u$ connected to $b_1$ with edge weight $x \geq 0 $ and connected with both $a_1$ and $a_2$ with edge weight $1$. That is, the weight matrix and initial labels are $$W = 
\begin{bmatrix}
    0 & 0 & 1 & 0 \\
    0 & 0 & 1 & 0 \\
    1 & 1 & 0 & x \\
    0 & 0 & x & 0
\end{bmatrix}, Y = \begin{bmatrix}
    -1 \\
    -1 \\
    0 \\
    1 
\end{bmatrix}.$$

The closed form solution is $$F^* = (S + \lambda I_n \Delta_{i \in L})^{-1} \lambda Y $$
where $S = \text{diag}(W \vec{1}_n) - W $. We now calculate: 
\begin{align*}
S &=
\begin{bmatrix}
    1 & 0 & -1 & 0 \\
    0 & 1 & -1 & 0 \\
    -1 & -1 & x + 2 & -x \\
    0 & 0 & -x & x
\end{bmatrix}
\\
S + \lambda I_n \Delta_{i \in L} &= \begin{bmatrix}
    1 + \lambda & 0 & -1 & 0 \\
    0 & 1+ \lambda & -1 & 0 \\
    -1 & -1 & x + 2 & -x \\
    0 & 0 & -x & x+ \lambda
\end{bmatrix}
\end{align*}
Recall that the prediction on the unlabeled point is $\hat y_3 = \text{sign} ([F*]_{32} - [F^*]_{31})$, so we calculate 
\begin{align*}
\hat y_3 = \text{sign}(F*]_{32} - [F^*]_{31})
=& \text{sign}\left(
-2\lambda\left(\frac{{\lambda + x}}{{\lambda^2x + 2\lambda^2 + 3\lambda x}}\right) + \lambda\left(\frac{{\lambda x + x}}{{\lambda^2x + 2\lambda^2 + 3\lambda x}}\right) \right)\\
=& \text{sign}\left(-2\lambda(\lambda + x) + \lambda(\lambda x + x) \right) \tag{since $\lambda > 0$ and $x \geq 0$}
\\=& \text{sign}\left(-2(\lambda + x) + (\lambda x + x) \right) \tag{since $\lambda > 0$}
\\=& \text{sign}\left(-2\lambda - x + \lambda x  \right)
\end{align*}
Solving the equation $-2\lambda - x + \lambda x = 0$, we know that the prediction changes and only change when $\lambda = \frac{x}{x-2}$. Let $x = \frac{2\lambda}{\lambda - 1} \geq 0$, then $\hat y_3 = -1$ when $\lambda < \lambda'$ and $\hat y_3 = 1$  when $\lambda \geq \lambda'$, which completes our proof. 
\end{proof}

The remaining proof is exactly the same as Lemma~\ref{lem:alpha alternating sign} and Theorem~\ref{thm:alpha upper bound}, by simply replacing notation $\alpha$ with $\lambda$.

\subsection{Proof of Theorem~\ref{thm:delta upper bound}}\label{appendix:delta upper bound}
\paragraph{Upper Bound.}
Using \Cref{lem: delta F form}, we know that each entry of $F^*$ is
$$F^*_{ij}(\delta) = \frac{1}{\det(I- c \cdot S)}\sum_{k=1}^n [\text{adj}(I- c \cdot S)]_{ik}Y_{kj} = \frac{1}{\det(I- c \cdot S)}\sum_{k=1}^n (a_{ik}Y_{kj})\exp(\delta \ln (d_{i}^{-1}d_{k})).$$ 

Recall that the prediction on a node is made by $\hat y_i = \text{argmax}(F^*_i
)$, so the prediction changes only when 
\begin{align*}
    F^*_{ic_1} - F^*_{ic_2}
    &= \frac{1}{\det(I- c \cdot S)}\left(\sum_{k=1}^n (a_{ik}Y_{kc_1})\exp(\delta \ln (d_{i}^{-1}d_{k})) - \sum_{k=1}^n (a_{ik}Y_{kc_2})\exp(\delta \ln (d_{i}^{-1}d_{k}))\right) \\
    &= \frac{1}{\det(I- c \cdot S)}\left( \sum_{k=1}^n (a_{ik}(Y_{kc_1}-Y_{kc_2}))\exp(\delta \ln (d_{i}^{-1}d_{k}))\right)\\
    &= 0.
\end{align*} By Lemma~\ref{lem:roots of exp sum}, $F^*_{ic_1} - F^*_{ic_2}$ has at most $n-1$ roots, so the prediction on node $i$ can change at most $n-1$ times. As $\delta$ vary, the prediction can change at most $\binom{c}{2} O(n) \in O(nc^2)$ times. For $n$ nodes and $m$ problem instances, this implies that we have at most $O(mn^2c^2)$ distinct values of loss. The pseudo-dimension $m$ then satisfies $2^m \leq O(mn^2c^2)$, or $m = O(\log nc)$.

\paragraph{Lower Bound} We construct the small connected component as follows:

\begin{lemma} \label{lem:delta connected components}
Consider when $c \geq 1/2$. Given $x \in [\log(2c)/\log(2),1)$, there exists a labeling instance $(G, L)$ with $4$ nodes, such that the predicted label of the unlabeled points changes only at $\delta = x$ as $\delta$ varies in $(0,1)$.
\end{lemma}

\begin{proof}
    We use binary labeling $a$ and $b$. We have two points labeled $a$ (namely $a_1, a_2$), and one point labeled $b$ (namely $b_1$) connected with both $a_1$ and $a_2$ with edge weight $1$. We also have an unlabeled point $u$ connected to $b_1$ with edge weight $x \geq 0$. That is, the affinity matrix and initial labels are $$W = \begin{bmatrix}
    0 & 1 & 1 & x\\
    1 & 0 & 0 & 0\\
    1 & 0 & 0 & 0\\
    x & 0 & 0 & 0
\end{bmatrix}, Y = \begin{bmatrix}
    1 & 0 \\
    0 & 1 \\
    0 & 1 \\
    0 & 0 
\end{bmatrix}.$$
Recall that the score matrix is $$F^* = (I - c \cdot S)^{-1}Y ,$$ 
where $S = D^{-\delta}WD^{\delta - 1}$ and $D$ is diagonal with $D_{ii} = \sum_i W_{ij}$. We now calculate: 
\begin{align*}
S = D^{-\delta}WD^{\delta - 1} =&
\begin{bmatrix}
0 & (x+2)^{-\delta} & (x+2)^{-\delta} & x^{\delta}(x+2)^{-\delta} \\
(x+2)^{-\delta} & 0 & 0 & 0\\
(x+2)^{-\delta} & 0 & 0 & 0\\
x^{\delta}(x+2)^{-\delta} & 0 & 0 & 0
\end{bmatrix},
\end{align*}
\begin{align*}
\det(I-c \cdot S) 
&= \det \begin{bmatrix}
    1 & -c(x+2)^{-\delta} & -c(x+2)^{-\delta} & -cx^{\delta}(x+2)^{-\delta} \\
    -c(x+2)^{-\delta} & 1 & 0 & 0\\
    -c(x+2)^{-\delta} & 0 & 1 & 0\\
    -cx^{\delta}(x+2)^{-\delta} & 0 & 0 & 1
\end{bmatrix} \\
&= 1-c^2 \neq 0,
\end{align*}
so $(I-c\cdot S)$ is invertible on our instance.

Recall that the prediction on the unlabeled point is $\hat y_4 = \text{argmax}F^*_4$, so we calculate 
\begin{align*}
\hat y_4 = \text{sign} (F^*_{4, 2} - F^*_{4, 1}) 
= \text{sign}\left(
\frac{c \cdot x ^{1-\delta}(2c - (x+2)^{\delta})}
{(1-c^2)(x+2) }\right)
= \text{sign}\left(2c - (x+2)^\delta\right). \tag{since $c \in (0,1)$, and $x \geq 0$}
\end{align*}
Solving the equation $2c-(x+2)^\delta = 0$, we know that the prediction changes and only change when $\delta = \frac{\ln(2c)}{\ln(x+2)}$. Since $x \leq \ln(2c)/\ln(2) \leq 1$, we can let $x = \left(2c\right)^{1/x} - 2 \geq 0$, then $\hat y_4 = 0$ when $\alpha < x$ and $\hat y_4 = 1$  when $\alpha \geq x$, which completes our proof. 
\end{proof}

\section{Introduction to GAT and GCN} \label{appendix:GAT_GCN}
Here, we provide a brief introduction to GAT and GCN.

\paragraph{Graph Convolutional Neural Networks (GCNs)} The fundamental idea behind GCNs is to repeatedly apply the convolution operator on graphs~\citep{kipf2016semi}.
Define $h_i^{0} = z_i$ as the input feature of the $i$-th node and let $h_i^{\ell}$ be the feature of the $\ell$-th layer of the $i$-th node.
We have the following update rule for the features of $h_i^{\ell}$
\begin{align*}
    h_i^{\ell} 
    =  \sigma \left( \sum_{j \in \mathcal{N}_i}  \frac{1}{\sqrt{d_i d_j }} U^{\ell-1} h_j^{\ell-1}\right)
\end{align*}
where $d_i$ represents the degree of vertex $i$, $U^{\ell}$ represents the learnable weights in our model, $\mathcal{N}_i$ represents the neighbors of vertex $i$, and $\sigma(\cdot)$ is the activation function. 

\paragraph{Graph Attention Neural Networks (GATs)}
GAT  is a more recent architecture that leverages the self-attention mechanisms to capture the importance of neighboring nodes to generate the features of the next layer~\citep{velivckovic2017graph}.
One of the advantages of GAT is its ability to capture long-range dependencies within the graph while giving more weight to influential nodes. This makes GAT particularly effective for tasks involving irregular graph structures and tasks where global context is essential.

Different from GCN, GAT uses the update rule for each layer
\[ h_i^{\ell} = \sigma \left(\sum_{j \in \mathcal{N}_i} e_{ij}^{\ell-1} U^{\ell-1}  h_j^{\ell-1}\right), \]
where 
\begin{align}
e_{ij}^{\ell} = \frac{\exp(\hat{e}_{ij}^{\ell})}{\sum_{j' \in \mathcal{N}_i} \exp(\hat{e}_{ij'}^{\ell})}, 
~~
\hat{e}_{ij}^{\ell} = \sigma\left(V^{\ell} [U^{\ell}h_i^{\ell}, U^{\ell}h_j^{\ell}]\right). \label{eqn:gat}
\end{align}
Here $e_{ij}^{\ell}$ is the attention score of node $j$ for node $i$ and $V^{\ell}$ and $U^{\ell}$ are learnable parameters.

\section{Proofs in Section~\ref{sec:gcn}}

We provide additional proof details from Section~\ref{sec:gcn} below. 

\subsection{Proof of Theorem~\ref{thm:rc of SGC}}\label{appendix:SGC}
\begin{lemma}
The $l_2$ norm of different embedding vectors produced by $(\beta, \theta), (\beta^\prime, \theta^\prime)$ after they process the tree all the way from the leaf level to the root can be bounded as 
\begin{align*}
    \Delta_{L, i} 
    \leq& \left(
    \frac{C_{dl}^2 + C_{dh}^2 + C_{dh}}{C_{dl}^3}
    \right)\|T_{L-1, i}(\beta, \theta)\| \|\beta - \beta^\prime\| + \left(\frac{1}{C_{dl}+1} + \frac{C_{dh}}{C_{dl}}\right)\Delta_{L-1, i}
\end{align*}
\end{lemma}

\begin{proof}
    \begin{align*}
    \Delta_{L, i} 
    &= \|T_{L, i}(\beta, \theta) - T_L(\beta^\prime, \theta^\prime)\|\\
    &= \Bigg\|\left(\frac{\beta}{d_i + \beta}T_{L-1, i}(\beta, \theta) + \sum_{j=1}^n \frac{w_{ij} T_{L-1, j}(\beta, \theta)}{\sqrt{(d_i+\beta)(d_j + \beta)}}\right) \\
    &\qquad- \left(\frac{\beta^\prime}{d_i + \beta^\prime}T_{L-1, i}(\beta^\prime, \theta^\prime) + \sum_{j=1}^n \frac{w_{ij} T_{L-1, j}(\beta^\prime, \theta^\prime)}{\sqrt{(d_i+\beta^\prime)(d_j + \beta^\prime)}}\right)\Bigg\|\\
    &\leq \left\|\left(\frac{\beta}{d_i + \beta}T_{L-1, i}(\beta, \theta)-\frac{\beta^\prime}{d_i + \beta^\prime}T_{L-1, i}(\beta^\prime, \theta^\prime)\right)\right\| \\
    &\qquad+ \sum_{j=1}^n \left(\|w_{ij}\| \left\|\left( \frac{T_{L-1, j}(\beta, \theta)}{\sqrt{(d_i+\beta)(d_j + \beta)}} - \frac{ T_{L-1, j}(\beta^\prime, \theta^\prime)}{\sqrt{(d_i+\beta^\prime)(d_j + \beta^\prime)}}\right)\right\|\right) \tag{by triangle inequality}\\
\end{align*}
The first part can be bounded as
\begin{align*}
    & \left\|\frac{\beta}{d_i + \beta}T_{L-1, i}(\beta, \theta)-\frac{\beta^\prime}{d_i + \beta^\prime}T_{L-1, i}(\beta^\prime, \theta^\prime)\right\|\\
    &\qquad\leq  \left\| 
    \frac{\beta}{d_i + \beta}T_{L-1, i}(\beta, \theta)
    -\frac{\beta^\prime}{d_i + \beta^\prime}T_{L-1, i}(\beta, \theta)\right\| \\
    &\qquad\qquad+\left\|\frac{\beta^\prime}{d_i + \beta^\prime}T_{L-1, i}(\beta, \theta)
    - \frac{\beta^\prime}{d_i + \beta^\prime}T_{L-1, i}(\beta^\prime, \theta^\prime) 
    \right\| \tag{by triangle inequality}\\
    &\qquad\leq  \left\| \frac{\beta}{d_i + \beta} - \frac{\beta^\prime}{d_i + \beta^\prime} \right\|\left\|T_{L-1, i} (\beta,\theta)\right\|
    +\left\|\frac{\beta^\prime}{d_i + \beta^\prime}\right\|
    \Delta_{L-1, i}\tag{by Cauchy-Schwarz inequality}
\end{align*}
Since $\beta \in [0,1]$ and $d_i \in [C_{dl}, C_{dh}]$, we have
\begin{align*}
    \left\|\frac{\beta^\prime}{d_i + \beta^\prime}\right\|
    = \frac{\beta^\prime}{d_i + \beta^\prime}
    \leq \frac{1}{C_{dl}+1},
\end{align*}
and
\begin{align*}
    \left\| \frac{\beta}{d_i + \beta} - \frac{\beta^\prime}{d_i + \beta^\prime} \right\|
    = \left\|\frac{d_i(\beta - \beta^\prime)}{(d_i+\beta)(d_i + \beta^\prime)}\right\|
    \leq \|\beta-\beta^\prime\| \frac{1}{C_{dl}}.
\end{align*}

For the second term, let's consider each element in the summation. Using a similar method as above, we get
\begin{align*}
    &\left\|\frac{T_{L-1, j}(\beta, \theta)}{\sqrt{(d_i+\beta)(d_j + \beta)}} - \frac{ T_{L-1, j}(\beta^\prime, \theta^\prime)}{\sqrt{(d_i+\beta^\prime)(d_j + \beta^\prime)}}\right\|\\
    &\qquad\leq\left\|\frac{T_{L-1, j}(\beta, \theta)}{\sqrt{(d_i+\beta)(d_j + \beta)}} 
    - \frac{ T_{L-1, j}(\beta, \theta)}{\sqrt{(d_i+\beta^\prime)(d_j + \beta^\prime)}}\right\|\\
    &\qquad\qquad+ \left\|\frac{T_{L-1, j}(\beta, \theta)}{\sqrt{(d_i+\beta^\prime)(d_j + \beta^\prime)}} 
    - \frac{ T_{L-1, j}(\beta^\prime, \theta^\prime)}{\sqrt{(d_i+\beta^\prime)(d_j + \beta^\prime)}}\right\| \tag{by triangle inequality}\\
    &\qquad\leq \left\|\frac{1}{\sqrt{(d_i+\beta)(d_j + \beta)}} 
    - \frac{1}{\sqrt{(d_i+\beta^\prime)(d_j + \beta^\prime)}}\right\|\|T_{L-1, j}(\beta, \theta)\| \\
    &\qquad\qquad+ \left\|\frac{1}{\sqrt{(d_i+\beta^\prime)(d_j + \beta^\prime)}}\right\|\Delta_{L-1, i} \tag{Cauchy-Schwarz inequality}
\end{align*}
Using the bounds on $\beta$ and $d_i$, we have 
\begin{align*}
    \left\|\frac{1}{\sqrt{(d_i+\beta^\prime)(d_j + \beta^\prime)}}\right\| \leq \frac{1}{C_{dl}},
\end{align*}
and
\begin{align*}
     &\left\|\frac{1}{\sqrt{(d_i+\beta)(d_j + \beta)}} 
    - \frac{1}{\sqrt{(d_i+\beta^\prime)(d_j + \beta^\prime)}}\right\|\\
    &\qquad=  \left\|  \frac{(d_i+\beta)(d_j + \beta) - (d_i+\beta^\prime)(d_j + \beta^\prime)}{\sqrt{(d_i+\beta)(d_j + \beta)(d_i+\beta^\prime)(d_j + \beta^\prime)}[\sqrt{(d_i+\beta)(d_j + \beta)} + \sqrt{(d_i+\beta^\prime)(d_j + \beta^\prime)}]} \right\|
    \\
    &\qquad\leq  \left\|  \frac{(d_i+d_j + \beta + \beta^\prime)(\beta-\beta^\prime)}{C_{dl}+\beta)(C_{dl}+\beta^\prime)[(C_{dl}+\beta)+ (C_{dl}+\beta^\prime)]} \right\|
    \\
    &\qquad\leq   \frac{C_{dh}+1}{C_{dl}^3} \|\beta-\beta^\prime\|
\end{align*}
Combining these results together, we get
\begin{align*}
    \Delta_{L, i} 
    &\leq \frac{1}{C_{dl}}\|\beta - \beta^\prime\|\|T_{L-1, i}(\beta, \theta)\| + \frac{1}{C_{dl}+1}\Delta_{L-1, i} \\
    &\qquad+ \sum_{i=1}^n \left(\|w_{ij}\|\left(
    \frac{(C_{dh}+1)\|T_{L-1, i}(\beta, \theta)\|}{C_{dl}^3}
    \|\beta - \beta^\prime\| 
    + \frac{1}{C_{dl}}\Delta_{L-1,i}\right)\right)\\
    &= \frac{1}{C_{dl}}\|\beta - \beta^\prime\|\|T_{L-1, i}(\beta, \theta)\| + \frac{1}{C_{dl}+1}\Delta_{L-1, i} \\
    &\qquad+ d_i\left(
    \frac{(C_{dh}+1)\|T_{L-1, i}(\beta, \theta)\|}{C_{dl}^3}
    \|\beta - \beta^\prime\| + \frac{1}{C_{dl}}\Delta_{L-1,i}\right)\\
    &\leq \left(\frac{1}{C_{dl}}+ 
    \frac{(C_{dh}+1)C_{dh}}{C_{dl}^3}
    \right)\|T_{L-1, i}(\beta, \theta)\| \|\beta - \beta^\prime\| + \left(\frac{1}{C_{dl}+1} + \frac{C_{dh}}{C_{dl}}\right)\Delta_{L-1, i}\\
    &\leq  \left(
    \frac{C_{dl}^2 + C_{dh}^2 + C_{dh}}{C_{dl}^3}
    \right)\|T_{L-1, i}(\beta, \theta)\| \|\beta - \beta^\prime\| + \left(\frac{1 + C_{dh}}{C_{dl}}\right)\Delta_{L-1, i}
\end{align*}
\end{proof}

\begin{lemma}
    The term $\|T_{L-1, i}(\beta, \theta)\| $ satisfies
    \begin{align*}
        \|T_{L-1, i}(\beta, \theta)\| \leq \left(\frac{\beta + C_{dh}C_z}{C_{dl} + \beta}\right)^L B_xB_\theta
    \end{align*}
\end{lemma}
\begin{proof}
\begin{align*}
    \|T_{L-1, i}(\beta, \theta)\|
    &= \left\|\frac{\beta}{d_i + \beta}T_{L-2, i}(\beta, \theta) + \sum_{j=1}^n \frac{w_{ij} T_{L-2, j}(\beta, \theta)}{\sqrt{(d_i+\beta)(d_j + \beta)}}\right\|\\
    &\leq \left\|\frac{\beta}{d_i + \beta}T_{L-2, i}(\beta, \theta)\right\| + \sum_{j=1}^n \left\|\frac{w_{ij} T_{L-2, j}(\beta, \theta)}{\sqrt{(d_i+\beta)(d_j + \beta)}}\right\| 
    \tag{by triangle inequality}\\
    &\leq \frac{\beta}{d_i + \beta}\|T_{L-2, i}(\beta, \theta)\| + \sum_{j=1}^n \|w_{ij}\|\left\|\frac{ T_{L-2, j}(\beta, \theta)}{\sqrt{(d_i+\beta)(d_j + \beta)}}\right\|
    \tag{by Cauchy-Schwarz}\\
    &\leq \frac{\beta}{d_i + \beta}\|T_{L-2, i}(\beta, \theta)\| + C_{dh} \max_{j}\left\|\frac{T_{L-2, j}(\beta, \theta)}{\sqrt{(d_i+\beta)(d_j + \beta)}}\right\| \\
    &\leq \frac{\beta}{C_{dl} + \beta}\|T_{L-2, i}(\beta, \theta)\| + \frac{C_{dh}}{C_{dl}+\beta} \max_j \|T_{L-2, j}(\beta, \theta)\|\\
    &\leq \left(\frac{C_{dh} + \beta}{C_{dl} + \beta}\right)^{L-1} \|z_i \theta_i\| 
    \tag{by recursively bounding $\|T_{l, i}(\beta, \theta)\|$}\\
    &\leq \left(\frac{C_{dh}}{C_{dl}}\right)^{L-1} C_z C_\theta
\end{align*}
\end{proof}

\begin{lemma}
    The change in margin loss for each node, due to change in parameters, after $L$ layers is 
    \begin{align*}
        \Lambda_i \leq \frac{2}{\gamma} \left(\left(
        \frac{C_{dl}^2 + C_{dh}^2 + C_{dh}}{C_{dl}^3}
        \right)\left(\frac{C_{dh}}{C_{dl}}\right)^{L-1} C_z C_\theta \|\beta - \beta^\prime\| \cdot \frac{k_1 - k_1^L}{1-k_1} + k_1^L C_z \|\theta_i - \theta_i^\prime\|\right), 
    \end{align*}
    where $k_1 = (1+C_{dh}/C_{dl})$.
\end{lemma}
\begin{proof}
    From previous lemmas, we know how to recursively bound $\Delta_{L,i}$ using $\Delta_{L-1, i}$, but it remains for us to bound the base case $\Delta_{0, i}$. We have
    \begin{align*}
        \Delta_{0, i} 
        = \|T_{0, i}(\beta, \theta) - T_{0, i}(\beta, \theta)\|
        = \|z_i\theta_i - z_i\theta_i^\prime\|
        \leq \|z_i\| \|\Theta_i - \theta_i^\prime\| \leq C_z \|\theta_i - \theta_i^\prime\|,
    \end{align*}
    where the inequality is by Cauchy-Schwarz.
    For the simplicity of notation, let $\bar T_L$ be the bound we derived for $\|T_{L-1, i}(\beta, \theta)\|$ from the previous lemma. We have
    \begin{align*}
        \Delta_{L, i} 
        &\leq \left(
        \frac{C_{dl}^2 + C_{dh}^2 + C_{dh}}{C_{dl}^3}
        \right)\|T_{L-1, i}(\beta, \theta)\| \|\beta - \beta^\prime\| + \left(\frac{1 + C_{dh}}{C_{dl}}\right)\Delta_{L-1, i}\\
        &\leq \left(
        \frac{C_{dl}^2 + C_{dh}^2 + C_{dh}}{C_{dl}^3}
        \right)\bar T_{L} \|\beta - \beta^\prime\| + \left(\frac{1 + C_{dh}}{C_{dl}}\right)\Delta_{L-1, i}\\
        &= \left(
        \frac{C_{dl}^2 + C_{dh}^2 + C_{dh}}{C_{dl}^3}
        \right)\bar T_{L} \|\beta - \beta^\prime\| \cdot \sum_{l=0}^{L-1} \left(\frac{1 + C_{dh}}{C_{dl}}\right)^l + \left(\frac{1 + C_{dh}}{C_{dl}}\right)^L \cdot \Delta_{0, i}
        \tag{by recursively bounding the terms}\\
        &= \left(
        \frac{C_{dl}^2 + C_{dh}^2 + C_{dh}}{C_{dl}^3}
        \right)\bar T_{L} \|\beta - \beta^\prime\| \cdot \frac{k_1 - k_1^L}{1-k_1} + k_1^L C_z \|\theta_i - \theta_i^\prime\| 
    \end{align*}
    where 
    \begin{align*}
        k_1 &= \frac{1 + C_{dh}}{C_{dl}}.
    \end{align*}

    The change in margin loss for each node after $L$ layers is then 
    \begin{align*}
        \Lambda_i 
        &= \left|g_\gamma(-\tau (f_{\beta, \theta}(x_i), y_i)) - g_\gamma(-\tau (f_{\beta^\prime, \theta^\prime}(x_i), y_i))\right|\\
        &\leq \frac{1}{\gamma}\left|\tau (f_{\beta, \theta}(x_i), y_i)) - \tau (f_{\beta^\prime, \theta^\prime}(x_i), y_i))\right|
        \tag{since $g_\gamma$ is $1/\gamma$-Lipschitz}\\
        &= \frac{1}{\gamma}\left|(2f_{\beta, \theta}(x_i) - 1) y_i-(2f_{\beta^\prime, \theta^\prime}(x_i) - 1) y_i)\right|\\
        &\leq \frac{2}{\gamma}\left|y_i\right|\left|f_{\beta, \theta}(x_i) - f_{\beta^\prime, \theta^\prime}(x_i)\right|
        \tag{by Cauchy-Schwarz inequality}\\
        &\leq \frac{2}{\gamma}\left|\sigma(T_{L, i}(\beta, \theta)) - \sigma(T_{L, i}(\beta^\prime, \theta^\prime))\right|
        \tag{since $y_i \in \{-1,1\}$}\\
        &\leq \frac{2}{\gamma}\left|T_{L, i}(\beta, \theta) - T_{L, i}(\beta^\prime, \theta^\prime)\right| 
        \tag{since sigmoid is $1$-Lipschitz}\\
        &= \frac{2}{\gamma} \Delta_{L, i}\\
        &\leq \frac{2}{\gamma} \left(\left(
        \frac{C_{dl}^2 + C_{dh}^2 + C_{dh}}{C_{dl}^3}
        \right)\left(\frac{C_{dh}}{C_{dl}}\right)^{L-1} C_z C_\theta \left(\frac{k_1 - k_1^L}{1-k_1} \right) \|\beta - \beta^\prime\| + k_1^L C_z \|\theta_i - \theta_i^\prime\|\right) 
    \end{align*}
\end{proof}

\begin{lemma}
    The change in margin loss $\Lambda_i$ for each node can be bounded by $\epsilon$, using a covering of size P, where P depends on $\epsilon$.
\end{lemma} 
\begin{proof}
    Let $k_2 = \frac{2}{\gamma}\left(
        \frac{C_{dl}^2 + C_{dh}^2 + C_{dh}}{C_{dl}^3}
    \right)\left(\frac{C_{dh}}{C_{dl}}\right)^{L-1} C_z C_\theta \left(\frac{k_1 - k_1^L}{1-k_1} \right)$ and $k_3 = \frac{2}{\gamma} k_1^L C_z$ for simplicity of notation. 
    
    We begin by noting that we can find a covering $\mathcal{C}\left(\beta, \frac{\epsilon}{4k_2}, |\cdot|\right)$ of size 
    \begin{align*}
        \mathcal{N}\left(\beta, \frac{\epsilon}{4k_2}, |\cdot|\right) \leq \frac{8k_2}{\epsilon}+1.
    \end{align*}
   Also, we can find a covering $\mathcal{C}\left(\theta, \frac{\epsilon}{4k_3}, \|\cdot\|\right)$ of size 
    \begin{align*}
        \mathcal{N}\left(\theta, \frac{\epsilon}{4k_3}, \|\cdot\|\right) \leq \left(\frac{8k_3}{\epsilon}+1\right)^{d}.
    \end{align*}
    Thus, for any specified $\epsilon$, we can ensure that $\Lambda_i$ is at most $\epsilon$ with a covering number 
    \begin{align*}
        P \leq \mathcal{N}\left(\beta, \frac{\epsilon}{4k_2}, |\cdot|\right)\mathcal{N}\left(\Theta, \frac{\epsilon}{4k_3}, \|\cdot\|\right) 
        \leq \left(\frac{8\max\{k_2, k_3\}}{\epsilon} + 1\right)^{d+1}.
    \end{align*}
    When $\epsilon < 8\max\{k_2, k_3\}$, we have 
    \begin{align*}
        \log P \leq (d+1) \log \left(\frac{16\max\{k_2, k_3\}}{\epsilon}\right).
    \end{align*}
\end{proof}

We can now finish our proof for Lemma~\ref{thm:rc of SGC}.
\begin{proof}
    Using Lemma A.5 from \cite{bartlett2017spectrally}, we obtain that 
    \begin{align*}
        \hat R_{\mathcal{T}}(\mathcal{H}_{(\beta, \theta)}^\gamma)
        \leq \inf_{\alpha > 0}\left(\frac{4\alpha}{\sqrt{m}} + \frac{12}{m} \int_{\alpha}^{\sqrt{m}} \sqrt{\log \mathcal{N}(\mathcal{H}_{(\beta, \theta)}^\gamma, \epsilon, \|\cdot \|)} d\epsilon\right).
    \end{align*}
    Using the previous lemmas, we have
    \begin{align*}
        \int_{\alpha}^{\sqrt{m}} \sqrt{\log \mathcal{N}(\mathcal{H}_{(\beta, \theta)}^\gamma, \epsilon, \|\cdot \|)} d\epsilon
        &= \int_{\alpha}^{\sqrt{m}} \sqrt{\log P} d\epsilon\\
        &\leq \int_{\alpha}^{\sqrt{m}} \sqrt{(d+1)\log \left(\frac{16\max\{k_2, k_3\}}{\epsilon}\right)} d\epsilon\\
        &\leq \sqrt{m} \sqrt{(d+1) \log \left(\frac{16\max\{k_2, k_3\}}{\alpha}\right)}
    \end{align*}
    Plugging in $\alpha = \sqrt{\frac{1}{m}}$, we have
    \begin{align*}
        \hat R_{\mathcal{T}}(\mathcal{H}_{(\beta, \theta)}^\gamma)
        \leq \frac{4}{m} + \frac{12\sqrt{(d+1)\log(16\sqrt{m}\max\{k_2, k_3\})}}{\sqrt{m}}.
    \end{align*}
\end{proof}

\subsection{Proof of Theorem~\ref{thm:rc of GCAN}}\label{appendix:GCAN}
\begin{lemma}\label{lem:inequalities}
    For any $z, z^\prime, \in \mathbb{R}^{d \times r}$ and $b, b^\prime \in \mathbb{R}^{r \times t}$ such that $\|z\|_F \leq C_z, \|z^\prime \|_F\leq C_z, \|b\|_F\leq C_b, \|b^\prime\|_F \leq C_b$, we have 
    \begin{align*}
        \|zb - z^\prime b^\prime\|_F 
        \leq C_z\|b-b^\prime\|_F + C_b\|z-z^\prime \|_F.
    \end{align*}
    The result also holds when $z, b, z^\prime, b^\prime$ are vectors or real numbers. The corresponding norms are $\|\cdot \|$ and $|\cdot |$.

    Also, by recursively using the inequality above, we may have that for any $z_1, \ldots, z_n$ and $z_1^\prime, \ldots, z_n^\prime$ such that $\|z_i \|\leq C_i, \|z_i^\prime \|\leq C_i$,
    \begin{align*}
        \|z_1 z_2 \ldots z_n - z_1^\prime z_2^\prime \ldots z_n^\prime\| \leq \sum_{i=1}^n \left(\|z_i - z_i^\prime\| \prod_{j\in [n], j \neq i} C_j\right).
    \end{align*}
    Here, for simplicity of notation, we used $\|\cdot \|$ to denote the type of norm that corresponds to the dimension of the $z_i$'s.
\end{lemma}
\begin{proof}
    \begin{align*}
        \|ab - a^\prime b^\prime\|_F 
        &= \|ab - a^\prime b^\prime + ab^\prime - ab^\prime\|_F\\
        &\leq \|ab - ab^\prime\|_F + \|ab^\prime - a^\prime b^\prime\|_F \tag{by triangle inequality}\\
        &\leq \|a\|_F\|b-b^\prime\|_F + \|b^\prime\|_F\|a-a^\prime \|_F \tag{by Cauchy-Schwarz inequality}\\
        &\leq C_z\|b-b^\prime\|_F + C_b\|a-a^\prime \|_F
    \end{align*}
\end{proof}

\begin{lemma}
The $l_2$ norm of different embedding vectors at level $L$, $h_i^L$, produced by $(\alpha, U, V), (\alpha^\prime, U^\prime, V^\prime)$ after they process the tree all the way from the leaf level to the root can be bounded as
\begin{align*}
    \Delta_{i,L} 
    \leq & 
    C_U \left(\max_{j \in \mathcal{N}_i} \left\|h_j^{L-1}\right\|\right) |\eta - \eta^\prime|
    +  rC_U \left(\max_{j \in \mathcal{N}_i} \left\|h_j^{L-1}\right\|\right)
    + \left(\max_{j \in \mathcal{N}_i}\|h_j^{L-1}\|\right) \left\|U - U^{\prime}\right\| \\
    &+ C_U \left(\max_{j \in \mathcal{N}_i}\left\|h_j^{L-1} - h_j^{\prime (L-1)}\right\|\right)
    + \frac{2rC_U}{C_{dl}}\left\|h_i^{L-1} - h_i^{\prime (L-1)}\right\|\\
    &+ \frac{2r}{C_{dl}} \left\|h_i^{L-1}\right\| \left\|U - U^{\prime}\right\| + \frac{2rC_U}{C_{dl}}\left|\eta - \eta^\prime\right|
\end{align*}
\end{lemma}
\begin{proof}
    \begin{align*}
        \Delta_{i, L}=
        &\left\|h_i^{L}(\eta, U, V) - h_i^{L}(\eta^\prime, U^\prime, V^\prime)\right\| \\
        =& \Big\|\sigma \left(\sum_{j \in \mathcal{N}_i}\left(\eta \cdot e_{ij}^{L-1} + (1 - \eta) \cdot \frac{1}{\sqrt{d_i d_j }}\right)U  h_j^{L-1}\right)\\
        &- \sigma \left(\sum_{j \in \mathcal{N}_i}\left(\eta^\prime \cdot e_{ij}^{\prime(L-1)}  + (1 - \eta^\prime) \cdot \frac{1}{\sqrt{d_i d_j }} \right) U^{\prime}  h_j^{\prime (L-1)} \right)\Big\|
        \\
        \leq & \Big\| \sum_{j \in \mathcal{N}_i} \left((\eta \cdot e_{ij}^{L-1} U  h_j^{L-1}) - (\eta^\prime \cdot e_{ij}^{\prime(L-1)} U^{\prime}  h_j^{\prime(L-1)})\right)
        \\
        &+ \sum_{j \in \mathcal{N}_i} \left((1 - \eta) \cdot \frac{1}{\sqrt{d_i d_j }} U h_j^{L-1} - (1 - \eta^\prime) \cdot \frac{1}{\sqrt{d_i d_j }} U^{\prime}  h_j^{\prime (L-1)}\right)\Big\|
        \tag{since $\sigma$ is $1$-Lipschitz}
        \\
        \leq & \sum_{j \in \mathcal{N}_i} \left\|(\eta \cdot e_{ij}^{L-1} U  h_j^{L-1}) - (\eta^\prime \cdot e_{ij}^{\prime(L-1)} U^{\prime}  h_j^{\prime(L-1)})\right\|\\
        &+ \sum_{j \in \mathcal{N}_i} \left\|(1 - \eta) \cdot \frac{1}{\sqrt{d_i d_j }} U  h_i^{L-1} - (1 - \eta^\prime) \cdot \frac{1}{\sqrt{d_i d_j }} U^{\prime }  h_i^{\prime (L-1)}\right\|
        \tag{by triangle inequality}
    \end{align*}
    Using Lemma~\ref{lem:inequalities}, we can bound each term in the first summation as
    \begin{align*}
         &\left\|(\eta \cdot e_{ij}^{L-1} U  h_j^{L-1}) - (\eta^\prime \cdot e_{ij}^{\prime(L-1)} U^{\prime}  h_j^{\prime(L-1)})\right\|\\
         &\qquad\leq  C_U \bar e_{ij}^{L-1} \bar h_j^{L-1} \cdot |\eta - \eta^\prime|
         + C_U \bar h_j^{L-1}\cdot \left|e_{ij}^{L-1} - e_{ij}^{\prime (L-1)}\right|\\
         &\qquad\qquad+ \bar e_{ij}^{L-1} \bar h_j^{L-1} \left\|U - U^{\prime}\right\|
         + C_U \bar e_{ij}^{L-1} \left\|h_j^{L-1} - h_j^{\prime (L-1)}\right\|
    \end{align*}
    Here, $\bar h_j^{L-1}$ is an upper bound on $\|h_j^{L-1}\|$ and $\|h_j^{\prime (L-1)}\|$, and $\bar e_{ij}^{L-1}$ is an upper bound on $|e_{ij}^{L-1}|$ and $|e_{ij}^{\prime(L-1)}|$.
    
    Bounding each term in the second summation, we have 
    \begin{align*}
        &\left\|(1 - \eta) \cdot \frac{1}{\sqrt{d_i d_j }} U  h_i^{L-1} - (1 - \eta^\prime) \cdot \frac{1}{\sqrt{d_i d_j }} U^{\prime}  h_i^{\prime (L-1)}\right\|\\
        &\qquad\leq  \left\|\frac{1}{\sqrt{d_i d_j }} U  h_i^{L-1} - \frac{1}{\sqrt{d_i d_j }} U^{\prime}  h_i^{\prime (L-1)}\right\| 
        + \left\|\eta \cdot \frac{1}{\sqrt{d_i d_j }} U  h_i^{L-1} -\eta^\prime \cdot \frac{1}{\sqrt{d_i d_j }} U^{\prime }  h_i^{\prime (L-1)} \tag{by triangle inequality}\right\|\\
        &\qquad\leq  \frac{1}{C_{dl}} \|U  h_i^{L-1} - U^{\prime}  h_i^{\prime (L-1)}\| 
        + \frac{1}{C_{dl}} \|\eta \cdot U  h_i^{L-1} -\eta^\prime \cdot U^{\prime }  h_i^{\prime (L-1)}\|\\
        &\qquad\leq \frac{1}{C_{dl}}\left(C_U \|h_i^{L-1} - h_i^{\prime (L-1)}\| + \bar h_i^{L-1} \|U - U^{\prime }\|\right)\\
        &\qquad\qquad+ \frac{1}{C_{dl}}\left(C_U \|h_i^{L-1} - h_i^{\prime (L-1)}\| + \bar h_i^{L-1} \|U - U^{\prime}\|
        + C_U \bar h_i^{L-1} |\eta - \eta^\prime|\right)\tag{using Lemma~\ref{lem:inequalities}}\\
        &\qquad=  \frac{1}{C_{dl}}\left(2C_U \|h_i^{L-1} - h_i^{\prime (L-1)}\| + 2\bar h_i^{L-1} \|U - U^{\prime}\|
        + C_U \bar h_i^{L-1} |\eta - \eta^\prime|\right).
    \end{align*}
    
    Combining the above results, we have
    \begin{align*}
        \Delta_i^L 
        \leq & \sum_{j \in \mathcal{N}_i} \Big( C_U \bar e_{ij}^{L-1} \bar h_j^{L-1} \cdot |\eta - \eta^\prime|
        + C_U \bar h_j^{L-1}\cdot |e_{ij}^{L-1} - e_{ij}^{\prime (L-1)}|\\
        &+ \bar e_{ij}^{L-1} \bar h_j^{L-1} \|U - U^{\prime}\|
        + C_U \bar e_{ij}^{L-1} \|h_j^{L-1} - h_j^{\prime (L-1)}\|\\
        &+ \frac{1}{C_{dl}}\big(2C_U \|h_i^{L-1} - h_i^{\prime (L-1)}\| + 2\bar h_i^{L-1} \|U - U^{\prime}\|
        + C_U \bar h_i^{L-1} |\eta - \eta^\prime|\big) \Big)
        \\
        \leq& C_U (\max_{j \in \mathcal{N}_i}\bar h_j^{L-1}) |\eta - \eta^\prime|
        +  rC_U (\max_{j \in \mathcal{N}_i}\bar h_j^{L-1})
        + (\max_{j \in \mathcal{N}_i}\bar h_j^{L-1}) \|U - U^{\prime}\| \\
        &+ C_U (\max_{j \in \mathcal{N}_i}\|h_j^{L-1} - h_j^{\prime (L-1)}\|)
        + \frac{2rC_U}{C_{dl}}\|h_i^{L-1} - h_i^{\prime (L-1)}\|\\
        &+ \frac{2r}{C_{dl}} \bar h_i^{L-1} \|U - U^{\prime}\| + \frac{2rC_U}{C_{dl}}|\eta - \eta^\prime|
        \tag{since $e_{ij}^\ell \leq 1$, $\sum_{j \in \mathcal{N}_i} e_{ij}^\ell = 1$, and the branching factor is $r$}
    \end{align*}
    It remains for us to derive $\bar h_j^{L-1}$ for all $j$.
\end{proof}

\begin{lemma}
    We can upper bound the norm of node feature embedding at level $\ell+1$ by
    \begin{align*}
         \|h_i^\ell\|
        \leq r^\ell C_U^{\ell+1} C_z \max\left(1, \frac{1}{C_{dl}}\right)^\ell.
    \end{align*}
\end{lemma}
\begin{proof}
    \begin{align*}
        \|h_i^{\ell+1}\|
        &= \left\|\sigma \left(\sum_{j \in \mathcal{N}_i} (\eta \cdot e_{ij}^{\ell} + (1 - \eta) \cdot \frac{1}{\sqrt{d_i d_j }} )  U  h_j^{\ell} \right)\right\|\\
        &\leq \left\|\sum_{j \in \mathcal{N}_i} (\eta \cdot e_{ij}^{\ell} + (1 - \eta) \cdot \frac{1}{\sqrt{d_i d_j }} )  U  h_j^{\ell} \right\|
        \tag{since $\|\sigma(x)\| \leq \|x\|$}\\
        &\leq \sum_{j \in \mathcal{N}_i} \left|\eta \cdot e_{ij}^{\ell} + (1 - \eta) \cdot \frac{1}{\sqrt{d_i d_j }} \right| \| U\|  \|h_j^{\ell} \|
        \tag{by triangle inequality and Cauchy-Schwarz inequality}\\
        &\leq C_U  \sum_{j \in \mathcal{N}_i} \left|\eta \cdot e_{ij}^{\ell} + (1 - \eta) \cdot \frac{1}{\sqrt{d_i d_j }} \right| \|h_j^{\ell} \|\\
        &\leq r C_U \max\left(1, \frac{1}{C_{dl}}\right)\left(\max_{j \in \mathcal{N}_i}\|h_j^{\ell-1}\|\right)
    \end{align*}
    Recursively bounding the terms, we have 
    \begin{align*}
        \|h_i^\ell\| \leq r^\ell C_U^\ell \max\left(1, \frac{1}{C_{dl}}\right)^\ell \max_{j \in [n]}\|h_j^0\|
        \leq r^\ell C_U^{\ell+1} C_z \max\left(1, \frac{1}{C_{dl}}\right)^\ell.
    \end{align*}
\end{proof}

\begin{lemma}
    The change in margin loss due to the change in parameter values after L layers satisfies
    \begin{align*}
        \Lambda_{i} \leq  \frac{2}{k} \left(k_1 + k_2|\eta-\eta^\prime| + k_3 \|U - U^{\prime}\|\right) \frac{k_4^L - k_4}{k_4 - 1} + k_4 C_z \|U - U^\prime\|,
    \end{align*}
    where 
    \begin{align*}
        k_1 &= r^L C_U^{L+1} C_z \max\left(1, \frac{1}{C_{dl}}\right)^{L-1}\\
        k_2 &= r^{L-1} C_U^{L+1} C_z \max\left(1, \frac{1}{C_{dl}}\right)^{L-1} + \frac{2rC_U}{C_{dl}}\\
        k_3 &= \left(1 + \frac{2r}{C_{dl}}\right)r^{L-1} C_U^{L} C_z \max(1, \frac{1}{C_{dl}})^{L-1}\\
        k_4 &= C_U + \frac{2rC_U}{C_{dl}}.
    \end{align*}
\end{lemma}
\begin{proof}
    Using the previous two lemmas, we know
    \begin{align*}
        &\|h_i^{L}(\eta, U, V) - h_i^{L}(\eta^\prime, U^\prime, V^\prime)\|\\
        &\qquad\leq C_U (\max_{j \in \mathcal{N}_i}\bar h_j^{L-1}) |\eta - \eta^\prime|
        +  rC_U (\max_{j \in \mathcal{N}_i}\bar h_j^{L-1})
        + (\max_{j \in \mathcal{N}_i}\bar h_j^{L-1}) \|U - U^{\prime}\| \\
        &\qquad\qquad+ C_U (\max_{j \in \mathcal{N}_i}\|h_j^{L-1} - h_j^{\prime (L-1)}\|)
        + \frac{2rC_U}{C_{dl}}\|h_i^{L-1} - h_i^{\prime (L-1)}\|
        + \frac{2r}{C_{dl}} \bar h_i^{L-1} \|U - U^{\prime}\| + \frac{2rC_U}{C_{dl}}|\eta - \eta^\prime|\\
        &\qquad\leq  k_1 + k_2|\eta-\eta^\prime| + k_3 \|U - U^{\prime}\| + k_4 (\max_{j \in [n]}\|h_j^{L-1} - h_j^{\prime (L-1)}\|)
        \\
        &\qquad= \left(k_1 + k_2|\eta-\eta^\prime| + k_3 \|U - U^{\prime}\|\right) \frac{k_4^L - k_4}{k_4 - 1} + k_4 (\max_{j \in [n]}\|h^0_j - h_j^{\prime 0}\|)\\
        &\qquad\leq  \left(k_1 + k_2|\eta-\eta^\prime| + k_3 \|U - U^{\prime}\|\right) \frac{k_4^L - k_4}{k_4 - 1} + k_4 C_z \|U - U^\prime\|
    \end{align*}
    where 
    \begin{align*}
        k_1 &= r^L C_U^{L+1} C_z \max(1, \frac{1}{C_{dl}})^{L-1}\\
        k_2 &= r^{L-1} C_U^{L+1} C_z \max(1, \frac{1}{C_{dl}})^{L-1} + \frac{2rC_U}{C_{dl}}\\
        k_3 &= \left(1 + \frac{2r}{C_{dl}}\right)r^{L-1} C_U^{L} C_z \max(1, \frac{1}{C_{dl}})^{L-1}\\
        k_4 &= C_U + \frac{2rC_U}{C_{dl}}.
    \end{align*}
    The change in margin loss for each node after L layers is then
    \begin{align*}
        \Lambda_i 
        &= \left|g_\gamma(-\tau (f_{\eta, U,V}(x_i), y_i)) - g_\gamma(-\tau (f_{\eta^\prime, U^\prime,V^\prime}(x_i), y_i))\right|\\
        &\leq \frac{1}{\gamma}\left|\tau (f_{\eta, U,V}(x_i), y_i)) - \tau (f_{\eta^\prime, U^\prime,V^\prime}(x_i), y_i))\right|
        \tag{since $g_\gamma$ is $1/\gamma$-Lipschitz}\\
        &= \frac{1}{\gamma}\left|(2f_{\beta, \theta}(x_i) - 1) y_i-(2f_{\beta^\prime, \theta^\prime}(x_i) - 1) y_i)\right|\\
        &\leq \frac{2}{\gamma}\left|y_i\right|\left|f_{\eta, U,V}(x_i) - f_{\eta^\prime, U^\prime,V^\prime}(x_i)\right|
        \tag{by Cauchy-Schwarz inequality}\\
        &\leq \frac{2}{\gamma}\left|\sigma(h_i^{L}(\eta, U, V)[0]) - \sigma(h_i^{L}(\eta^\prime, U^\prime, V^\prime)[0])\right|
        \tag{since $y_i \in \{-1,1\}$}\\
        &\leq \frac{2}{\gamma}\left|h_i^{L}(\eta, U, V)[0] - h_i^{L}(\eta^\prime, U^\prime, V^\prime)[0]\right| 
        \tag{since $\sigma$is $1$-Lipschitz}\\
        % &\leq \frac{2}{\gamma} \Delta_{L, i}\\
        &\leq \frac{2}{\gamma} \left(k_1 + k_2|\eta-\eta^\prime| + k_3 \|U - U^{\prime}\|\right) \frac{k_4^L - k_4}{k_4 - 1} + k_4 C_z \|U - U^\prime\|.
    \end{align*}
\end{proof}

\begin{lemma}
    The change in margin loss $\Lambda_i$ for each node can be bounded by $\epsilon$, using a covering of size $P$, where $P$ depends on $\epsilon$, with 
    \begin{align*}
        \log P \leq (d^2+1)\log \left(\frac{8\max\{A, BC_U\sqrt{d}\}}{\epsilon}\right).
    \end{align*}  
\end{lemma}
\begin{proof}
    We let $A = \frac{2k_2(k_4^L - k_4)}{k (k_4 - 1)}$ and $B = \frac{2k_3(k_4^L-k_4) + \gamma(k_4^2-k_4)C_z}{\gamma(k_4-1)}$ for simplicity of notation. Note that we have $\Lambda_i \leq A|\eta - \eta^\prime | + B \|U - U^\prime \|$.
    
    We begin by noting that we can find a covering $\mathcal{C}(\eta, \frac{\epsilon}{2A}, |\cdot |)$ of size
    \begin{align*}
        \mathcal{N}(\eta, \frac{\epsilon}{2A}, |\cdot |) \leq 1 + \frac{4A}{\epsilon}.
    \end{align*}
    We can also find a covering $\mathcal{C}(U, \frac{\epsilon}{2B}, \|\cdot \|_F)$ with size 
    \begin{align*}
        \mathcal{N}(U, \frac{\epsilon}{2B}, \|\cdot \|_F) \leq \left(1 + \frac{4BC_U\sqrt{d}}{\epsilon} \right)^{d^2}.
    \end{align*}
    For any specified $\epsilon$, we can ensure that $\Lambda_i$ is at most $\epsilon$ with a covering number of 
    \begin{align*}
        P 
        \leq&
        \mathcal{N}(\eta, \frac{\epsilon}{2A}, |\cdot |) \cdot \mathcal{N}(U, \frac{\epsilon}{2B}, \|\cdot \|_F)\\
        \leq&
        \left( 1 + \frac{4A}{\epsilon}\right)\left(1 + \frac{4BC_U\sqrt{d}}{\epsilon} \right)^{d^2} 
        \leq
        (1 + \frac{4\max\{A, BC_U\sqrt{d}\}}{\epsilon})^{d^2+1}
    \end{align*}
    Moreover, when $\epsilon \leq 4\max\{A, BC_U\sqrt{d}\}$, we have 
    \begin{align*}
        \log P \leq (d^2+1)\log \left(\frac{8\max\{A, BC_U\sqrt{d}\}}{\epsilon}\right).
    \end{align*}    
\end{proof}

Now we can finish our proof for Theorem~\ref{thm:rc of GCAN}.

\begin{proof}
    Using Lemma A.5 from \cite{bartlett2017spectrally}, we obtain that 
    \begin{align*}
        \hat R_{\mathcal{T}}(\mathcal{H}^\gamma_{(\eta, U, V))})
        \leq \inf_{\alpha > 0}\left(\frac{4\alpha}{\sqrt{m}} + \frac{12}{m} \int_{\alpha}^{\sqrt{m}} \sqrt{\log \mathcal{N}(\mathcal{H}^\gamma_{(\eta, U, V))}, \epsilon, \|\cdot \|)} d\epsilon\right).
    \end{align*}
    Using the previous lemmas, we have
    \begin{align*}
        \int_{\alpha}^{\sqrt{m}} \sqrt{\log \mathcal{N}(\mathcal{H}^\gamma_{(\eta, U, V))}, \epsilon, \|\cdot \|)} d\epsilon
        &= \int_{\alpha}^{\sqrt{m}} \sqrt{\log P} d\epsilon \\
        &\leq \int_{\alpha}^{\sqrt{m}} \sqrt{(d^2+1)\log \left(\frac{8\max\{A, BC_U\sqrt{d}\}}{\epsilon}\right)} d\epsilon\\
        &\leq \sqrt{m}\sqrt{(d^2+1)\log \left(\frac{8\max\{A, BC_U\sqrt{d}\}}{\alpha}\right)}
    \end{align*}
    Plugging in $\alpha = \sqrt{\frac{1}{m}}$, we have
    \begin{align*}
        \hat R_{\mathcal{T}}(\mathcal{H}^\gamma_{(\eta, U, V))})
        \leq \frac{4}{m} + \frac{12\sqrt{(d^2+1)\log \left(8\sqrt{m}\max\{A, BC_U\sqrt{d}\}\right)}}{\sqrt{m}}.
    \end{align*}
\end{proof}

\section{Experiments}\label{appendix:experiments}
\subsection{Label Propagation-based Method: Normalized Adjacency Matrix-Based Algorithmic Family}
We empirically validate our findings in \Cref{sec:label_prop}. For each of the eight datasets, the number of nodes per problem instance, $n$, is fixed at $30$.  We set the target generalization error to $\eps=0.1$,  and calculate the required number of problem instances as $m = O(\log n/\eps^2) \approx 300$. To evaluate performance, we randomly sample 300 graphs with 30 nodes each, tune the hyperparameter values to maximize accuracy on these graphs, and then test the selected hyperparameter on a separate set of 300 randomly sampled graphs. The results of evaluating the Normalized Adjacency Matrix-Based Algorithmic Family is presented in \Cref{tab:label_prob}, confirming that the observed generalization error is well within the scale of the target value $0.1$ (our bounds are somewhat conservative).

\begin{table*}[ht]
    \centering
    \begin{tabular}{|c|c|c|c|c|c|c|c|c|}
    \hline
         & CIFAR10 & WikiCS & CORA & Citeseer & PubMed & AmazonPhotos & Actor\\  
     \hline
         Train Acc. & 0.9445 & 0.7522 & 0.7927 &0.7845  &   0.9993 & 0.9983 &0.9185\\
    \hline
         Test Acc. & 0.9397 & 0.7485 & 0.8010 & 0.7714  &     0.9993 & 0.9989 & 0.9239\\
    \hline
        Abs. Diff. & 0.0048&  0.0037  &  0.0083 & 0.0131 &   0.    & 0.0006 & 0.0054\\
    \hline
    \end{tabular}
    % }
    \caption{The Training Accuracy and Testing Accuracy of learning the hyperparameter $\delta$ in Normalized Adjacency Matrix Based Family ($\mathcal{F}_\delta$). The absolute difference between the accuracies (i.e. generalization error) is well within the scale of our target value $0.1$. }
    \label{tab:label_prob}
\end{table*}

\subsection{GCAN Experiments}

In this section, we empirically evaluate our proposed GCAN interpolation methods on nine standard benchmark datasets. Our goal is to see whether tuning $\eta$ gives better results than both GCN and GAT. 
The setup details of our experiment are described in \Cref{appendix:datasets}.

\begin{table*}[ht]
\small
\centering
\resizebox{\textwidth}{!}{%
\begin{tabular}{|c|*{13}
{>{\centering\arraybackslash}p{1.3cm}|}}
\hline
\textbf{Dataset} & $0.0$ & $0.1$ & $0.2$ & $0.3$ & $0.4$ & $0.5$ & $0.6$ & $0.7$ & $0.8$ & $0.9$ & $1.0$ &  {Rel. GCN} &  {Rel. GAT}\\
\hline
CIFAR10 & $0.7888 \pm 0.0010$ 
&$0.7908 \pm 0.0008$ 
&$0.7908 \pm 0.0015$ 
&$0.7907 \pm 0.0012$ 
&$0.7943 \pm 0.0022$ 
&$0.7918 \pm 0.0018$ 
&$0.7975 \pm 0.0017$ 
&$0.7971 \pm 0.0023$ 
&$0.7921 \pm 0.0023$ 
&$0.7986 \pm 0.0028$ 
&$\textbf{0.7984}  \boldsymbol{\pm} \textbf{0.0023}$ 
&  {$4.54\%$}
&  {$0\%$}\\
\hline
WikiCS &$0.9525 \pm 0.0007$ 
&$0.9516 \pm 0.0006$ 
&$0.9532 \pm 0.0011$ 
&$0.9545 \pm 0.0008$ 
&$0.9551 \pm 0.0015$ 
&$0.9545 \pm 0.0012$ 
&$0.9539 \pm 0.0012$ 
&$\textbf{0.9553} \boldsymbol{\pm} \textbf{0.0012}$ 
&$0.9530 \pm 0.0007$ 
&$0.9536 \pm 0.0009$ 
&$0.9539 \pm 0.0009$
&  {$5.89\%$}
&  {$3.04\%$}
\\
\hline
Cora & $0.6132 \pm 0.0218$ 
&$0.8703 \pm 0.0251$ 
&$0.8879 \pm 0.0206$ 
&$0.8396 \pm 0.0307$ 
&$0.8022 \pm 0.0385$ 
&$0.8615 \pm 0.0402$ 
&$ \textbf{0.9011} \boldsymbol{\pm} \textbf{0.0421}$ 
&$0.8088 \pm 0.0362$ 
&$0.8505 \pm 0.0240$ 
&$0.8549 \pm 0.0389$ 
&$0.8725 \pm 0.0334$ 
&  {74.43\%}
&  {22.43\%}
\\
\hline
Citeseer & $\textbf{0.7632} \boldsymbol{\pm} \textbf{0.0052}$ 
&$0.6944 \pm 0.0454$ 
&$0.7602 \pm 0.0566$ 
&$0.7500 \pm 0.0461$ 
&$0.7339 \pm 0.0520$ 
&$0.7427 \pm 0.0462$ 
&$0.7588 \pm 0.0504$ 
&$0.7193 \pm 0.0567$ 
&$0.7661 \pm 0.0482$ 
&$0.7266 \pm 0.0412$ 
&$0.7471 \pm 0.0444$ 
&  {0\%}
&  {6.37\%
}
\\
\hline
PubMed & $0.9350 \pm 0.0009$ 
&$0.9306 \pm 0.0006$ 
&$0.9356 \pm 0.0009$ 
&$0.9281 \pm 0.0007$ 
&$\textbf{0.9356} \boldsymbol{\pm} \textbf{0.0007}$ 
&$0.9319 \pm 0.0009$ 
&$0.9313 \pm 0.0007$ 
&$0.9288 \pm 0.0009$ 
&$0.9313 \pm 0.0006$ 
&$0.9338 \pm 0.0010$ 
&$0.9356 \pm 0.0009$ 
&  {0.92\%}
&  {0\%}
\\ \hline
CoauthorCS & $0.9733 \pm 0.0007$ 
&$0.9733 \pm 0.0008$ 
&$\textbf{0.9765} \boldsymbol{\pm} \textbf{0.0005}$ 
&$0.9744 \pm 0.0005$ 
&$0.9733 \pm 0.0009$ 
&$0.9690 \pm 0.0007$ 
&$0.9712 \pm 0.0009$ 
&$0.9722 \pm 0.0005$ 
&$0.9722 \pm 0.0011$ 
&$0.9722 \pm 0.0007$ 
&$0.9744 \pm 0.0007$ 
&  {$11.99\%$}
&  {$08.20\%$}
\\ \hline
AmazonPhotos & $0.9605 \pm 0.0022$ 
&$0.9617 \pm 0.0007$ 
&$0.9629 \pm 0.0015$ 
&$0.9599 \pm 0.0013$ 
&$0.9641 \pm 0.0017$ 
&$0.9574 \pm 0.0018$ 
&$0.9641 \pm 0.0019$ 
&$0.9592 \pm 0.0133$ 
&$\textbf{0.9653} \boldsymbol{\pm} \textbf{0.0027}$ 
&$0.9635 \pm 0.0031$ 
&$0.9562 \pm 0.0019$ 
&  {$12.15\%$}
&  {$20.77\%$}
\\ \hline
Actor & $0.5982 \pm 0.0016$ 
&$0.5919 \pm 0.0022$ 
&$\textbf{0.6005} \boldsymbol{\pm} \textbf{0.0039}$ 
&$0.5959 \pm 0.0039$ 
&$0.5965 \pm 0.0038$ 
&$0.5970 \pm 0.0027$ 
&$0.5976 \pm 0.0037$ 
&$0.5993 \pm 0.0043$ 
&$0.5930 \pm 0.0041$ 
&$0.5970 \pm 0.0037$ 
&$0.5953 \pm 0.0031$ 
& {0.57\%}
& {1.28\%}
\\ \hline
Cornell & $0.7341 \pm 0.0097$ 
&$0.7364 \pm 0.0165$ 
&$0.7364 \pm 0.0073$ 
&$0.7205 \pm 0.0154$ 
&$0.7523 \pm 0.0109$ 
&$0.7795 \pm 0.0120$ 
&$0.7568 \pm 0.0188$ 
&$0.7500 \pm 0.0140$ 
&$0.7477 \pm 0.0138$ 
&$0.7909 \pm 0.0136$ 
&$\textbf{0.8000} \boldsymbol{\pm} \textbf{0.0423}$ 
& {24.78\%}
& {0\%}
\\ \hline
Wisconsin & $0.8688 \pm 0.0077$ 
&$\textbf{0.8922} \boldsymbol{\pm} \textbf{0.0035}$ 
&$0.8688 \pm 0.0080$ 
&$0.8906 \pm 0.0049$ 
&$0.8797 \pm 0.0044$ 
&$0.8578 \pm 0.0120$ 
&$0.8875 \pm 0.0037$ 
&$0.8781 \pm 0.0082$ 
&$0.8563 \pm 0.0128$ 
&$0.8750 \pm 0.0121$ 
&$0.8719 \pm 0.0076$
&  {17.84\%}
&  {15.84\%}
\\ \hline
\end{tabular}
}
\caption{Results on the proposed GCAN interpolation. Each column corresponds to one $\eta$ value. Each row corresponds to one dataset. Each entry shows the accuracy and the interval. The accuracy with optimal $\eta$ value outperforms both pure GCN and pure GAT. The right two columns show the percentage of prediction error reduction relative to GCN and GAT.} 
\label{tab:gat_gcn_table}
\end{table*}

In Table \ref{tab:gat_gcn_table}, we show the mean accuracy across $30$ runs of each $\eta$ value and the $90\%$ confidence interval associated with each experiment. It is interesting to note that for various datasets we see varying optimal $\eta$ values for best performance. More often than not, the best model is interpolated between GCN and GAT, showing that we can achieve an improvement on both baselines simply by interpolating between the two. For example, GCN achieves the best accuracy among all interpolations in Citeseer, but in other datasets such as CIFAR 10 or Wisconsin, we see higher final accuracies when the $\eta$ parameter is closer to $1.0$ (more like GAT). The interpolation between the two points also does not increase or decrease monotonically for many of the datasets. The optimal $\eta$ value for each dataset can be any value between $0.0$ and $1.0$. This suggests that one should be able to learn the best $\eta$ parameter for each specific dataset. By learning the optimal $\eta$ value, we can outperform both GAT and GCN. 

% \section{Additional Experiment Details}\label{appendix:datasets}

\subsection{Experiment Setup for GCAN}\label{appendix:datasets}
We apply dropout with a probability of $0.4$ for all learnable parameters, apply 1 head of the specialized attention layer (with new update rule), and then an out attention layer. The activation we choose is eLU activation (following prior work \citep{velivckovic2017graph}), with 8 hidden units, and 3 attention heads. 
% We start training with an initial learning rate of $7 \times 10^{-5}$ and a weight decay of $5 \times 10^{-4}$.

These GCAN interpolation experiments are all run with only $20\%$ of the dataset being labeled datapoints, and the remaining $80\%$ representing the unlabeled datapoints that we test our classification accuracy on. \Cref{tab:gat_gcn_hyperparameter} notes the exact setup of each dataset, and the overall training time of each experiment. We would like to examine our theory with the simplest network that is still non-linear, so we selected a hidden dimension being 1. Note that our theory on sample complexity bounds still applies to larger networks, but implementing our techniques on larger networks and larger graphs might require additional computational improvements.

\begin{table}[ht!]
\centering
\resizebox{\textwidth}{!}{%
\begin{tabular}{lcccccccc}
\toprule
Dataset     & Num of train nodes & learn rate & Epoch & Num of exp & Train time(sec) & Dim of hid. layers & Num of Attention Heads \\
\midrule
CiFAR10                        & 400                   & 7e-3          & 1000   & 30         & 13.5354               & 1                   & 3                 \\
WikiCS                        & 192                   & 7e-3          & 1000   & 30         & 6.4742               & 1                   & 3                 \\
Cora                        & 170                   & 7e-3          & 1000   & 30         & 7.4527               & 1                   & 3                 \\
Citeseer                        & 400                   & 7e-3          & 1000   & 30         & 6.4957               & 1                   & 3                 \\
Pubmed                        & 400                   & 7e-3          & 1000   & 30         & 13.1791               & 1                   & 3                 \\
CoAuthor CS       & 400                   & 0.01         & 1000   & 30         & 6.8015               & 1                   & 3                 \\

Amazon Photos       & 411                   & 0.01          & 400   & 30         & 11.0201              & 1                   & 3                \\
Actor       & 438                   & 0.01          & 1000   & 30         & 14.7753              & 1                   & 3                 \\
Cornell       & 10                   & 0.01          & 1000   & 30         & 6.9423              & 1                   & 3                 \\
Wisconsin       & 16                   & 0.01          & 1000   & 30         & 6.9271              & 1                   & 3                 \\
\end{tabular}%
}
\caption{Details of the datasets and experimental setup.}
\label{tab:gat_gcn_hyperparameter}
\end{table}

For datasets that are not inherently graph-structured (e.g., CIFAR-10), we first compute the Euclidean distance between the feature vectors of each pair of nodes. An edge is then added between two nodes if their distance is below a predefined threshold.

\end{document}